\documentclass{article}

\usepackage{microtype}
\usepackage{graphicx}
\usepackage{subcaption}
\usepackage{booktabs} 
\usepackage[dvipsnames]{xcolor}

\usepackage{hyperref}
\usepackage{natbib}

\newcommand{\bx}{{\boldsymbol{x}}}

\newcommand{\bv}{{\boldsymbol{v}}}
\newcommand{\bgamma}{{\boldsymbol{\gamma}}}
\newcommand{\balpha}{{\boldsymbol{\alpha}}}
\newcommand{\btheta}{{\boldsymbol{\theta}}}
\newcommand{\bSigma}{{\boldsymbol{\Sigma}}}
\newcommand{\bmu}{{\boldsymbol{\mu}}}

\newcommand{\rn}{\mathbb{R}}

\usepackage{tikz}
\usetikzlibrary{shapes.geometric}
\usetikzlibrary {shapes.arrows}

\DeclareRobustCommand{\reddot}{\tikz[baseline=-0.6ex]\node[draw=black,fill=red,circle,inner sep=2pt]{};}
\DeclareRobustCommand{\purplestar}{\tikz[baseline=-0.6ex]\node[draw=black,fill=Orchid,star,inner sep=1.5pt]{};}

\DeclareRobustCommand{\greendot}{\tikz[baseline=-0.6ex]\node[draw=black,fill=darkgreen,circle,inner sep=2pt]{};}

\DeclareRobustCommand{\bluedot}{\tikz[baseline=-0.6ex]\node[draw=black,fill=Cerulean,circle,inner sep=2pt]{};}

\DeclareRobustCommand{\kleinbluedot}{\tikz[baseline=-0.6ex]\node[draw=blue,fill=blue,circle,inner sep=2pt]{};}

\DeclareRobustCommand{\purpledot}{\tikz[baseline=-0.6ex]\node[draw=Orchid,fill=Orchid,circle,inner sep=2pt]{};}

\DeclareRobustCommand{\orangedot}{\tikz[baseline=-0.6ex]\node[draw=black,fill=orange,circle,inner sep=2pt]{};}

\DeclareRobustCommand{\turquoisedot}{\tikz[baseline=-0.6ex]\node[draw=black,fill=Turquoise,circle,inner sep=2pt]{};}
\DeclareRobustCommand{\periwinkledot}{\tikz[baseline=-0.6ex]\node[draw=black,fill=Periwinkle,circle,inner sep=2pt]{};}
\DeclareRobustCommand{\purpleline}{%
  \tikz[baseline=-0.6ex]\draw[violet, line width=1.5pt] (0,0)--(0.3,0);%
}

\DeclareRobustCommand{\redline}{%
  \tikz[baseline=-0.6ex]\draw[red, line width=1.5pt] (0,0)--(0.3,0);%
}
\DeclareRobustCommand{\greenline}{%
  \tikz[baseline=-0.6ex]\draw[darkgreen, line width=1.5pt] (0,0)--(0.3,0);%
}

\DeclareRobustCommand{\orangeline}{%
  \tikz[baseline=-0.6ex]\draw[orange, line width=1.5pt] (0,0)--(0.3,0);%
}

\DeclareRobustCommand{\grayline}{%
  \tikz[baseline=-0.6ex]\draw[gray, line width=1.5pt] (0,0)--(0.3,0);%
}

\usepackage{subcaption}
\usepackage{wrapfig}
\usepackage{bbm}
\usepackage{graphicx}
\usepackage[utf8]{inputenc} 
\usepackage[T1]{fontenc}    
\usepackage{hyperref}       
\usepackage{url}            
\usepackage{booktabs}       
\usepackage{amsfonts}       
\usepackage{nicefrac}       
\usepackage{microtype}      
\usepackage{xcolor}         
\usepackage{amsthm}
\usepackage{amsmath}

\definecolor{darkgreen}{rgb}{0,0.5,0}

\usepackage[preprint]{icml2026}

\usepackage{amsmath}
\usepackage{amssymb}
\usepackage{mathtools}
\usepackage{amsthm}

\usepackage[capitalize,noabbrev]{cleveref}

\theoremstyle{plain}
\newtheorem{theorem}{Theorem}[section]

\theoremstyle{definition}
\newtheorem{definition}[theorem]{Definition}

\theoremstyle{remark}

\usepackage[textsize=tiny]{todonotes}

\icmltitlerunning{Reducing Memorisation in Generative Models via Riemannian Bayesian Inference}

\begin{document}

\twocolumn[
  \icmltitle{Reducing Memorisation in Generative Models \\ via Riemannian Bayesian Inference}

  \icmlsetsymbol{equal}{*}

  \begin{icmlauthorlist}
    \icmlauthor{Johanna Marie Gegenfurtner}{equal,yyy}
    \icmlauthor{Albert Kj{\o}ller Jacobsen}{equal,yyy} 
    
    \icmlauthor{Naima Elosegui Borras}{yyy}
    \icmlauthor{Alejandro Valverde Mahou}{yyy}
    
    \icmlauthor{Georgios Arvanitidis}{yyy}

  \end{icmlauthorlist}

  \icmlaffiliation{yyy}{Department of Applied Mathematics and Computer Science, Technical University of Denmark}

  \icmlcorrespondingauthor{Johanna Marie Gegenfurtner}{johge@dtu.dk}
  \icmlcorrespondingauthor{Albert Kj{\o}ller Jacobsen}{akjja@dtu.dk}

  \icmlkeywords{Bayesian Methods, Differential Geometry}

  \vskip 0.3in
]

\printAffiliationsAndNotice{\icmlEqualContribution}

\begin{abstract}
Modern generative models can produce realistic samples, however, balancing memorisation and generalisation remains an open problem.
We approach this challenge from a Bayesian perspective by focusing on the parameter space of flow matching and diffusion models and constructing a predictive posterior that better captures the variability of the data distribution. In particular, we capture the geometry of the loss using a Riemannian metric and leverage a flexible approximate posterior that adapts to the local structure of the loss landscape. This approach allows us to sample generative models that resemble the original model, but exhibit reduced memorisation. Empirically, we demonstrate that the proposed approach reduces memorisation while preserving generalisation. Further, we provide a theoretical analysis of our method, which explains our findings.
Overall, our work illustrates how considering the geometry of the loss enables effective use of the parameter space, even for complex high-dimensional generative models.
\end{abstract}

\section{Introduction}

The success of modern generative models has raised questions about their capacity to merely memorise data or generate beyond it. While it is essential that a generative model captures the data distribution, it is critical in several applications to avoid overfitting to specific training examples. For predictive models, overfitting to the training data can boost performance as seen from the double descent phenomenon, however, there is no such thing as benign overfitting in the context of generative modelling, as perfectly replicating training data points raises concerns about privacy. Our work focuses on diffusion models, for which the problem of memorisation has been extensively discussed in recent work 
\cite{liu2024generative}. 
In this paper, we raise the question:

\textit{Can we reduce memorisation in modern generative models through uncertainty on their parameters?}

\begin{figure}
    \centering
    \includegraphics[width=\linewidth]{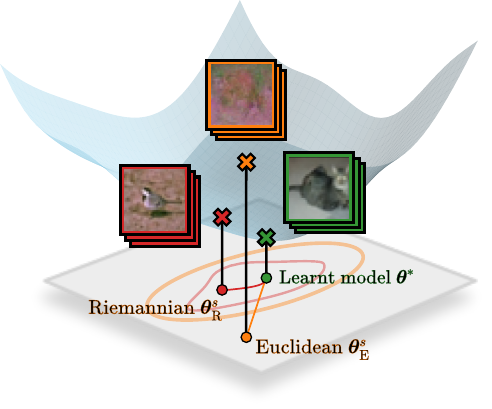}
    \caption{
    We propose reducing memorisation of a trained generative model by perturbing the learnt model parameters $\btheta^\ast$ (\greendot) by sampling from an approximate posterior distribution $\btheta \sim q\left(\btheta\right)$, thus constructing new models. While samples from a Euclidean approximate posterior (\orangedot) reduce memorisation, accounting for the geometry (\reddot) \textit{reduces memorisation without breaking the fit.}}
    \label{fig:figure1}
\end{figure}
To mitigate memorisation in modern generative models, we adopt a Bayesian treatment of the model parameters, using the Laplace approximation (LA) for defining a conceptually simple approximate posterior distribution. The Laplace approximation is defined as a Gaussian centred at the maximum a posteriori (MAP) estimate, however, it is often an overly crude approximation of the true posterior. To address this issue, we make use of the Riemannian Laplace approximation \cite{bergamin2023riemannian} that leverages the geometric structure of the true posterior. This allows us to perturb the parameters of the learnt model and construct an ensemble of generators, where each member reduces memorisation compared to the original model. As \citet{bertrand2025closed} have shown, generalisation happens exactly when the models fails to overfit to the training objective -- we achieve generalisation by breaking the overfitting behaviour.
Through experiments on flow matching and diffusion models, we demonstrate that introducing uncertainty on the model parameters while respecting the true posterior geometry is an effective way to reduce memorisation without forgetting how to generalise.

Specifically, our contributions are:
\begin{enumerate}
    \item We define a geometry-informed approximate posterior distribution over model parameters of diffusion-like generative models (Subsection \ref{subsec:RLA}).
    \item We extend a method for estimating generative uncertainty \citep{jazbec2025generative}. For deterministic generators, we define the exact likelihood of a generated sample (Subsection \ref{subsec:gen-u}).
    \item We provide mathematically grounded explanations on how our approach mitigates memorisation, and why a Riemannian perspective outperforms the standard Laplace approximation (Section \ref{sec:theory}).
    \item We provide empirical evidence that respecting the geometric structure can help generative models \textit{generalise rather than memorise}, as our posterior predictive is based on parameter samples from high-density regions of the true posterior (Section \ref{sec:experiments}). 
    
\end{enumerate}

\section{Background}

\subsection{Generative Modelling}

Assume we have $N$ data samples $\mathcal{D}_{\bx} =\{\bx_*^i\}_{i=1}^N\in\mathcal{X}\subseteq\mathbb{R}^D$ from an unknown distribution $p_*$. When training a generative model, the goal is to find a generator function $g_{\btheta}$ that maps samples from a known base distribution $\bx_0\sim p_0$ to new samples $\widehat{\bx}$ that approximately belong to $p_*$. Hence, the generative process can be summarised as:
\begin{equation}\label{eq:generator}
    \widehat{\bx} = g_{\btheta}\left(\bx_0\right), \qquad \bx_0 \sim p_0.
\end{equation}

\subsubsection{Flow Matching}
In flow matching, the generator's output is the endpoint of the solution of an initial value problem (IVP) of the form
\begin{equation}\label{eq:ODE-flow}
    \bx(0)=\bx_0, \qquad     \dot{\bx}(t)=u_\btheta(\bx_t,t),
\end{equation}
where $u_\btheta:\rn^D\times[0,1]\rightarrow\rn^D$ is a velocity field, represented by a neural network with parameters $\btheta$. We will denote a generated sample as $\widehat{\bx}=\bx(1).$ Usually, the base distribution is the unit Gaussian distribution $p_0= \mathcal{N}\left(\boldsymbol{0}, \mathbb{I}_D\right)$ and the IVP is solved with an Euler scheme. We denote the distribution of generated samples by $p\left(\widehat{\bx}\right).$

To learn the velocity field $u_\btheta,$ we optimise the loss function:
{\small
\begin{equation}
    \label{eq:loss}
    \mathcal{L}_{\operatorname{FM}}\left(\btheta, \mathcal{D}_{\bx}\right) = \mathbb{E}_{\substack{\bx_0 \sim p_0 \\\bx_\ast \sim \mathcal{D}_{\bx} \\ t\sim \mathcal{U}{\left[0,1\right]}}} \left[\lVert u_{\btheta}\left(\bx(t), t\right) - \left(\bx_{\ast} - \bx_{0}\right) \rVert^2_2\right].
\end{equation}}
For a Gaussian optimal transport path, a sample from the conditional probability path, $\bx_t \sim p_t\left(\cdot | \bx_\ast\right)$, is given by
\begin{equation}\label{eq:transport_path}
    \bx_t = t \bx_\ast + \left(1 - t\right) \bx_0.
\end{equation}
Differentiating with respect to time $t$ gives us the target function as in Equation \eqref{eq:loss}. For further reading, see the lecture notes by \citet{flowsanddiffusions2025}.

\subsubsection{Diffusion Models}
Under certain conditions, flow matching and denoising diffusion implicit models (DDIM) \citep{songdenoising} are equivalent at sampling time \citep{song2021maximum, gao2025diffusionmeetsflow}. A diffusion model defines a forward process that gradually transforms data into noise. This process is then reversed to generate samples. Given the extensive literature on diffusion models \cite{flowsanddiffusions2025, sohldickstein2015deepunsupervisedlearningusing}, we focus on DDIM and its relation to flow matching.

Let $p_t(\bx)$ denote the marginal distribution of a diffusion process at time $t \in [0,1]$, where $t=0$ corresponds to the data distribution and $t=1$ to a standard Gaussian. DDIM sampling can be formulated as the deterministic ODE
\begin{equation}
\label{eq:pf-ode}
    \dot{\bx}(t) = f(t)\bx(t) - \frac{1}{2} g(t)^2 \nabla_{\bx} \log p_t(\bx(t)),
\end{equation}
where $f,g:[0,1]\rightarrow\rn$ define the forward diffusion process and $\nabla_{\bx} \log p_t(\bx)$ is the score function, which is approximated by a neural network $s_\btheta(\bx(t),t)$.

Under mild regularity assumptions, 
\citet{gao2025diffusionmeetsflow} reparameterise Equation \eqref{eq:pf-ode} as a flow matching model:
\begin{equation}
\label{eq:velocity-ddim}
    u_\btheta(\bx(t),t) \equiv f(t)\bx(t) - \frac{1}{2} g(t)^2 s_\btheta(\bx(t),t).
\end{equation}

From this perspective, DDIM sampling is a particular parameterisation of a deterministic flow. 
This equivalence implies that samples from a diffusion process can be generated using a flow matching model with the appropriate parameterisation, and vice versa. 
For this reason, we adopt the flow matching formulation throughout.

\subsubsection{Memorisation in Generative Models}
In this subsection we will define the nearest neighbour memorisation measure.
We say that a point is memorised if it is much closer to one point from the training set than to all others \cite{yoon2023diffusion, buchanan2025edge}. 
\begin{definition}
    A generated data sample, $\widehat{\bx}$, is memorised if for a fixed threshold $c\in(0,1)$,
\begin{equation}
    \lVert \widehat{\bx} - \bx^{\left(1\right)}(\widehat{\bx})\rVert^2 \leq c \lVert \widehat{\bx} - \bx^{\left(2\right)}(\widehat{\bx})\rVert^2,
\end{equation}
where $\bx^{\left(1\right)}\left(\widehat{\bx}\right)$ and $\bx^{\left(2\right)}\left(\widehat{\bx}\right)$ are the closest and second-closest training samples to $\widehat{\bx}$, respectively.
\end{definition}
 The corresponding memorisation ratio \cite{buchanan2025edge} of a generative model $g_\btheta$ is given by
 {\small
\begin{equation}
    \label{eq:memorisation-ratio}    
\frac{1}{N} \sum_{i=1}^N \mathbbm{1} \left[\left\lVert \widehat{\bx}^i - \bx^{\left(1\right)}(\widehat{\bx}^i)\right\rVert^2 \leq c \left\lVert \widehat{\bx}^i - \bx^{\left(2\right)}(\widehat{\bx}^i)\right\rVert^2 \right],
\end{equation}}
where $\{ \widehat{\bx}^i = g_{\btheta}\left(\bx_0^i\right) \mid \bx_0^i \sim p_0  \}_{i=1}^N$ is a set of $N$ generated data points.
As we will discuss in Subsection \ref{subsec:memorisationmeasure}, there are several definitions of memorisation – e.g. by accounting for the average distance to the $K$ nearest training samples, and not only the first and second closest \citep{chen2024towards}.

\subsection{Bayesian Deep Learning}
\label{sec:bayesian} 
One way to quantify uncertainty of a neural network is to consider the variation in the outputs of an ensemble of learners. However, instead of training multiple neural networks, a common approach is to treat the model parameters in a \textit{Bayesian} manner by defining a posterior distribution over the model parameters using Bayes' rule. We can consider the negative log-posterior as a loss function given by
\begin{align}
    \mathcal{L}_{\text{post}}\left(\btheta, \mathcal{D}\right) &= - \log \frac{p\left(\mathcal{D}|\btheta\right)p\left(\btheta\right)}{p\left(\mathcal{D}\right)} \\
    \label{eq:lpost}
    &\propto - \log p\left(\mathcal{D} | \btheta\right) - \log p\left(\btheta\right).
\end{align}
Although the log-posterior is intractable, the fact that the log-joint is proportional to it allows us to define an approximate posterior distribution over the model parameters that respects the local geometry at the optimum. We can define this distribution \textit{post-hoc} using the Laplace approximation:
\begin{equation}  
    \label{eq:laplace}
    q\left(\btheta|\mathcal{D}\right)= \mathcal{N}\left(\btheta\mid\btheta^\ast, \mathbf{H}_{\mathcal{L}_{\text{post}}}^{-1}\left(\btheta^\ast\right)\right),
\end{equation}
where $\btheta^\ast=\arg\min_\btheta \mathcal{L}_{\text{post}}\left(\btheta\right)$ is the \textit{maximum a posteriori} (MAP) estimate and $\mathbf{H}_{\mathcal{L}}\left(\btheta\right) := \nabla_{\theta}^2 \mathcal{L}_{\text{post}}\left(\btheta, \mathcal{D}\right)$ is the Hessian of the posterior loss, which we evaluate at the MAP. We denote a sample from this Euclidean/classic Laplace approximation by $\btheta^{s}_{\text{E}} = \btheta^\ast + \bv$ where $\bv \sim \mathcal{N}\left(\boldsymbol{0}, \mathbf{H}_{\mathcal{L}}^{-1}(\btheta)\right)$.

\subsubsection{The Riemannian Laplace Approximation}\label{subsec:RLA}
\citet{bergamin2023riemannian} introduced a Riemannian take on the Laplace approximation, which outperforms the conventional method by respecting the geometric structure of the true posterior. This approach constrains samples from the approximate posterior to lie within high-density regions of the true posterior, while the standard Laplace approximation has no such guarantee. We explain a few concepts from the field of Riemannian geometry in Appendix \ref{app:geodesics}.

Mathematically, the Riemannian Laplace approximation $q_{\mathcal{M}}\left(\btheta| \mathcal{D}\right) = \left(f_{\mathcal{M}}\right)_{\#} q\left(\bv|\mathcal{D}\right)$ is the pushforward of a distribution of vectors $q\left(\bv|\mathcal{D}\right) = \mathcal{N}\left(\bv \mid \boldsymbol{0}, \mathbf{H}_{\mathcal{L}}^{-1}\left(\btheta^\ast\right)\right)$ through a mapping $f_\mathcal{M}$. This mapping depends on the geometry of the \emph{loss manifold}, especially on our choice of metric, which defines a way of measuring distances on the manifold. 

We define the (posterior) loss manifold as the graph of the posterior loss function given in Equation \eqref{eq:lpost}: \begin{equation}
    \mathcal{M} = \left\{h\left(\btheta\right) \mid \btheta\in\Theta\right\} \subseteq \mathbb{R}^{\left(K+1\right)}.
\end{equation}
A natural parametrisation of $\mathcal{M}$ is given by
\begin{equation}
    h:\Theta\rightarrow\rn^{K+1}, \quad h(\btheta)=\left(\theta_1, \dots, \theta_K,\mathcal{L}_{\text{post}}\left(\btheta\right)\right).
\end{equation}
This parametrisation has also been considered in \cite{ 
hartmann2022lagrangian, jacobsen2025monge, williams2025geodesic}, for instance.

For every $\bv \sim q\left(\bv|\mathcal{D}\right)$, there exists a curve $\balpha(t)\subseteq\Theta$ starting at $\balpha(0)=\btheta^*$ with initial velocity $\dot{\balpha}(0)=\bv$ such that $\bgamma(t)=\bigl(\balpha(t),\mathcal{L}(\balpha(t))\bigr)$ is a geodesic (a locally shortest path on the loss manifold $\mathcal{M}$), as illustrated in Figure \ref{fig:geodesic-figure}.

To sample from the Riemannian Laplace approximation, we compute the curve $\balpha(t)$ by solving the \emph{geodesic equation}: 
\begin{equation}
    \label{eq:geodesic-equation}
    \Ddot{\balpha}_k(t)=-\sum_{i,j=1}^n \dot{\balpha}_i(t)\dot{\balpha}_j(t)\cdot\Gamma_{ij}^k(\balpha(t)), 
\end{equation}
where $\Gamma_{ij}^k$ denote the \emph{Christoffel symbols}, and we set
\begin{equation}
    \label{eq:geodesic-initial-cond}
    \boldsymbol{\balpha}(0)=\btheta^\ast \qquad\text{and}\qquad \dot{\boldsymbol{\balpha}}(0) = \bv \sim q\left(\bv | \mathcal{D}\right)
\end{equation}
as the initial conditions. A sample from the Riemannian approximate posterior is then obtained by evaluating at $t=1$, hence $\btheta^{s}_{\text{R}} = \boldsymbol{\balpha}(1)$.
This is equivalent to computing the \emph{exponential map} 
followed by a projection to the parameter space, i.e. $\left(\btheta^{s}_{\text{R}}, \mathcal{L}\left(\btheta^{s}_{\text{R}}\right)\right)=\operatorname{Exp}_{h(\btheta^\ast)}(\mathbf{J}_h(\btheta^*)\bv)$. 

It should be noted that the Christoffel symbols depend on the choice of metric. 
We equip $\mathcal{M}$ with the \textit{pullback} metric as follows.
For two vectors $\boldsymbol{u}_1,\boldsymbol{u}_2\in\Theta$, we evaluate the scalar product at $h(\btheta)\in\mathcal{M}$ by computing 
\begin{eqnarray}
    \langle \boldsymbol{u}_1,\boldsymbol{u}_2\rangle_{\mathbf{G}(\btheta)}&=&\left(\mathbf{J}_h(\btheta) \boldsymbol{u}_1\right)^\top \left(\mathbf{J}_h(\btheta) \boldsymbol{u}_2\right)\\
    &=&\boldsymbol{u}_1^\top\underset{=\mathbf{G}(\btheta)}{\underbrace{\mathbf{J}_h(\btheta)^\top\mathbf{J}_h(\btheta)}} \boldsymbol{u}_2
\end{eqnarray}
where $\mathbf{J}_{h}\left(\btheta\right) \in \mathbb{R}^{\left(K+1\right)\times K}$ is the Jacobian of $h$ evaluated at $\btheta$ and the matrix $\mathbf{G}(\btheta)$ is given by 
\begin{equation}    
    \label{eq:monge-metric}
    \mathbf{G}(\btheta)=\mathbb{I}_K +\nabla_\theta\mathcal{L}(\btheta)\nabla_\theta\mathcal{L}(\btheta)^\top.
\end{equation}
Using this metric, the geodesic equation \eqref{eq:geodesic-equation} reduces to
\begin{equation}
    \label{eq:geodesic-equation-monge}
    \Ddot{\balpha}\left(t\right) = - \frac{\dot{\balpha}\left(t\right)^\top \mathbf{H}_{\mathcal{L}}\left(\balpha\left(t\right)\right) \dot{\balpha}\left(t\right)}{1 + \lVert \nabla_{\theta}\mathcal{L}\left(\balpha\left(t\right)\right)\rVert^2} \cdot \nabla_{\theta}\mathcal{L}\left(\balpha\left(t\right)\right),
\end{equation}
which can be vectorised and computed with automatic differentiation. \citet{yu2023riemannian} choose the Fisher metric, which is less efficient. \citet{jacobsen2025staying} take a similar approach in the context of input data augmentation. For further details, we refer to the textbook on differential geometry by \citet{Lee00} and the lecture notes by \citet{Gudmundsson2025_Riemann}.

\begin{figure}
    \centering
    \includegraphics[width=0.75\linewidth]{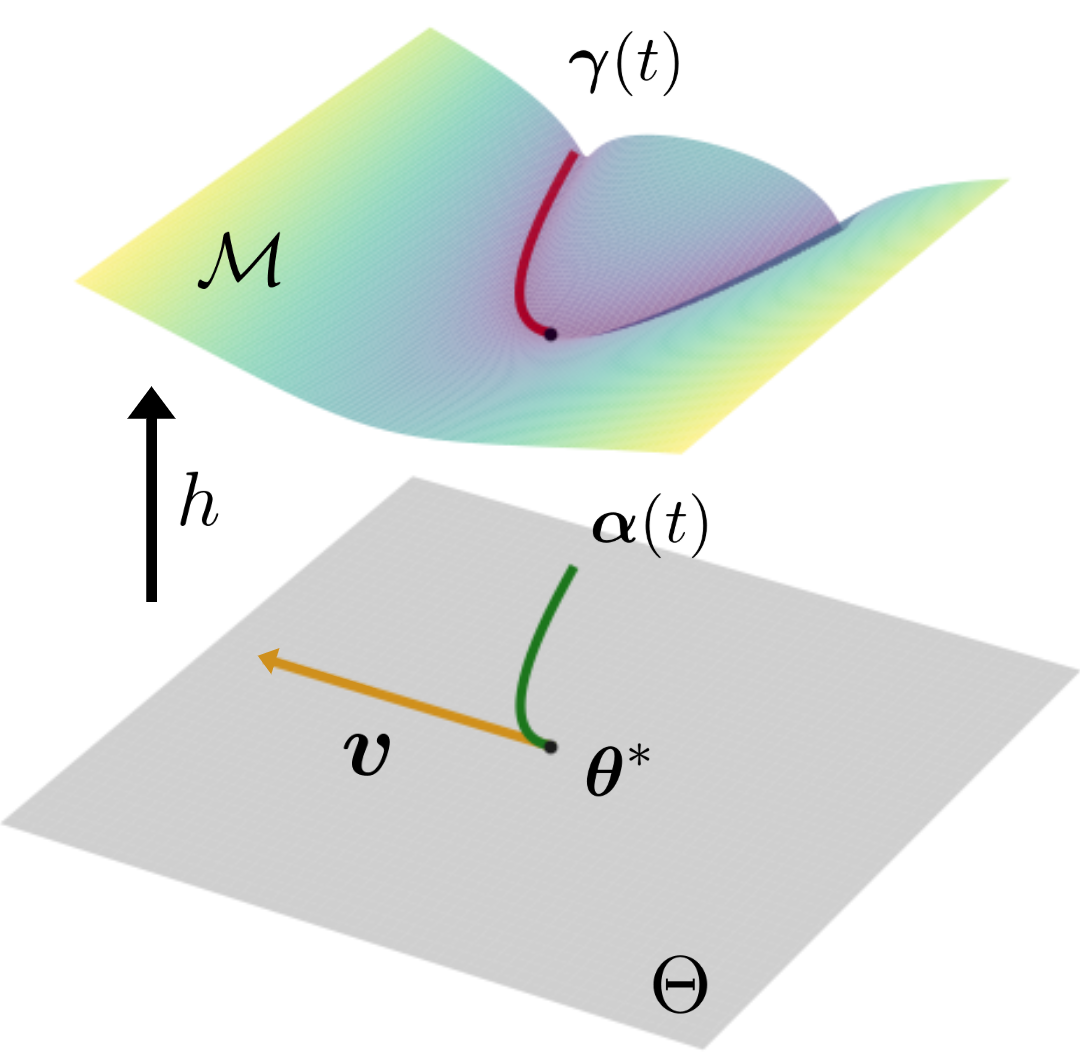}
    \caption{The flat parameter space $\Theta$ and a two-dimensional surface $h(\Theta)=\mathcal{M}.$ The arrow (\orangeline) represents the initial velocity vector $\bv$ of $\balpha(t)$ (\greenline) at the starting point $\btheta^*=\balpha(0)$, that is, $\dot{\balpha}(0)=\bv$. By mapping $\balpha(t)$ to $\mathcal{M}$ through $h$, we obtain the corresponding geodesic $\bgamma(t)=h(\balpha(t))$ (\redline) on $\mathcal{M}.$ In this specific example we obtain a geodesic path lying in a low loss region.}
    \label{fig:geodesic-figure}
\end{figure}

\section{Riemannian Bayesian Inference in Generative Models}

\begin{algorithm}[tb]
    \caption{Sampling the Riemannian Laplace Approx.}
    \label{alg:RLA-short}
    \begin{algorithmic}[1]
        \STATE \textbf{Input:} Generator $g_\btheta$, data $\mathcal{D}_{\bx}$, base distribution $p_0$, parametrisation $h$ of manifold $\mathcal{M}$
        \STATE \textbf{Output:} Samples $\widehat{\bx}$
        \vspace{0.5em}
        
        \STATE \textbf{Step 1:} Find the MAP estimate 
        \vspace{-5pt}
        \[
        \btheta^\ast = \arg\min_\theta \mathcal{L}(\btheta, \mathcal{D}_{\bx})
        \]
        \vspace{-15pt}
        \STATE \textbf{Step 2:} Compute (or approximate) the Hessian at $\btheta^\ast$  
        \vspace{-5pt}
        \[
        \mathbf{H}_{\mathcal{L}}\left(\btheta^\ast\right) = \nabla_\theta^2 \mathcal{L}(\btheta^\ast)
        \]
        \vspace{-16pt}
        \STATE \textbf{Step 3:} Sample initial velocity and map to $\mathcal{M}$ by
        \vspace{-5pt}
        \[
        \bv \sim \mathcal{N}(\mathbf{0}, \mathbf{H}_{\mathcal{L}}^{-1}\left(\btheta^\ast\right))
        \] 
        \vspace{-15pt}
        \[
        \bigl(\btheta^{s}_{\text{R}}, \mathcal{L}\left(\btheta^{s}_{\text{R}}\right)\bigr) = \operatorname{Exp}_{h(\btheta^\ast)}(\mathbf{J}_h(\btheta^\ast)\bv)
        \]
        \vspace{-13pt}
        \STATE \textbf{Step 4:} Generate new data using perturbed parameters  
        \vspace{-5pt}
        \[
        \bx_0 \sim p_0, \quad \widehat{\bx} = g_{\btheta^{s}_{\text{R}}}(\bx_0)
        \]
        \vspace{-15pt}
    \end{algorithmic}
\end{algorithm}

A diffusion-like generative model is by definition not based on a likelihood function. Although the loss objective of such models can be interpreted as a likelihood function if the per-time-step loss terms are sufficiently weighted across time \cite{jazbec2025generative, song2021maximum}, this is generally not the case. For this reason, we consider the Gibbs posterior (or generalised posterior), which replaces the negative log-likelihood in Equation \eqref{eq:lpost} with a (possibly weighted) loss function. In the case of flow matching, this generalised posterior loss function is given by:
\begin{equation}
    \widetilde{\mathcal{L}}_{\text{post}}\left(\btheta, \mathcal{D}_{\bx}\right) \propto \mathcal{L}_{\operatorname{FM}}\left(\btheta, \mathcal{D}_{\bx}\right) - \log p\left(\btheta\right).
\end{equation}
We will consider a uniform prior which is uninformative by definition since $\log p\left(\btheta\right)$ is constant. Consequently, the generalised posterior loss reduces to the loss function of the generative model. Our Riemannian Bayesian treatment of the generative model then follows a 4-step strategy described in Algorithm \ref{alg:RLA-short}. In contrast, the classic (non-Bayesian) generative model only follows Step 1 and Step 4. We provide a full algorithmic overview for large-scale experiments in Algorithm \ref{alg:riemannian-bayesian-inference} in Appendix \ref{app:image-experiments}. Note that we perturb the generator's parameters \emph{before} carrying out its sampling process. 

\subsection{Exact Likelihood and Generative Uncertainty}\label{subsec:gen-u}
The induced distribution of generated samples is the pushforward of the base distribution $p_0$. If the generator $g_{\btheta}$ is invertible and deterministic for a fixed $\btheta$ -- which is the case for flow matching -- we can further leverage the change-of-variables formula. This gives the explicit density:
\begin{equation}
    \label{eq:generative-likelihood}
    p\left(\widehat{\bx}|\btheta\right) = p_0\left(g_{\btheta}^{-1}\left(\widehat{\bx}\right)\right)\cdot |\det \mathbf{J}_{g_{\btheta}^{-1}} \left(\widehat{\bx}\right)|.
\end{equation}
Note that this construction allows us to directly compute the likelihood of a data point according to the learnt generator. Hence, we can express the full distribution of generated samples by marginalising over the Riemannian posterior:
\begin{equation}
    \label{eq:posterior-predictive}
    p\left(\widehat{\bx}|\mathcal{D}_{\bx}\right) = \mathbb{E}_{q_{\mathcal{M}}\left(\btheta|\mathcal{D}_{\bx}\right)}\left[ p(\widehat{\bx} | \btheta)\right].
\end{equation}
This is known as the \textit{posterior predictive distribution} and is well known in predictive modelling, as it forms a posterior distribution over functions rather than parameters. Measuring its variability, e.g. by using the entropy measure, provides a way to quantify uncertainty of the model.

\citet{jazbec2025generative} leverage this construction for generative models to propose the concept of \textit{generative uncertainty}. They formulate a semantic likelihood by encoding generated images to a semantic latent space (such as CLIP) and then placing a Gaussian distribution over the sample.
Our formulation of the likelihood of a generated image in Equation \eqref{eq:generative-likelihood} requires no such construction and can be directly computed, using only the generative model itself.

\section{Theoretical Insights}\label{sec:theory}
In this section we prove that parameter perturbations can reduce memorisation. We argue that this is due to the resulting changes of the learnt velocity field, and thus the generated data samples. We further show that the Riemannian method achieves better generalisation than its Euclidean counterpart.

\subsection{How Perturbations Impact the Vector Field}
\label{sub:perturbations-impact}
The parameter space perturbations induce new vector fields. For a fixed initial value, we analyse the trajectories and endpoints of a particle subject to these new vector fields. For the classic Laplace, there is a bias.
For small perturbations, and assuming that $u_\btheta$ is sufficiently smooth for all $(\bx,t)$ in our domain of interest, we can use a Taylor approximation:

{\small\begin{equation}
    u_{\btheta^{s}_\text{E}}(\bx(t), t)\approx u_{\btheta^\ast}(\bx(t), t)+(\btheta^{s}-\btheta^\ast)^\top\nabla_{\theta} u_{\btheta^\ast}(\bx(t),t).\label{eq:taylor-u}
\end{equation}}
Plugging this into Equation \eqref{eq:ODE-flow}, and using that $\btheta^{s}_{\text{E}}-\btheta^\ast=\bv\sim\mathcal{N}(\boldsymbol{0},\mathbf{\Sigma}),$
this induces a \emph{random ODE}:
\begin{align}\label{eq:RODE}
    &\bx(0,\bv)=\bx_0,\notag \\ &\dot{\bx}(t,\bv)=u_{\btheta^\ast}(\bx(t,\bv),t)+\bv^\top\nabla_\theta u_{\btheta^\ast}(\bx(t,\bv),t).
\end{align}
Due to the dependence on $\bv,$ we will denote the perturbed paths as 
$\bx(t,\bv).$ Each such random vector field as in Equation \eqref{eq:RODE} now induces a generator $g_{\btheta^\ast+\bv},$ which takes initial values $\bx_0\sim p_0$ to some $\widehat{\bx}.$ Thus, $g_{\btheta^\ast}(\bx_0)$ is the endpoint of the unperturbed trajectory.

Assuming that for a fixed initial value $\bx_0$ the mapping $$\varphi_{\bx_0}:\Theta\rightarrow\mathcal{X},\quad \varphi_{\bx_0}(\bv)=g_{\btheta^\ast+\bv}(\bx_0)$$ is smooth in $\bv,$ we may write: 
{\small\begin{align*}
      \varphi_{\bx_0}(\bv)\approx &g_{\btheta^\ast}(\bx_0)+\bv^\top\frac{\partial \varphi_{\bx_0}(\bv)}{\partial\bv}\vert_{\bv=0}+\frac{1}{2}\bv^\top\frac{\partial^2\varphi_{\bx_0}(\bv)}{\partial\bv^2}\vert_{\bv=0}\bv.
\end{align*}}
Thus, depending on these sensitivities, parameter perturbations can lead to perturbations of generated samples.
Consider the various endpoints of perturbed trajectories for a fixed starting point $\bx_0,$ we have that 

{\small
\begin{align*}
    \mathbb{E}_{q(\bv)}\left[\varphi_{\bx_0}(\bv)\right]
    &=g_{\btheta^\ast}(\bx_0)+\frac{1}{2}\mathbb{E}_{q(\bv)}\left[\bv^\top \frac{\partial^2\varphi_{\bx_0}(\bv)}{\partial\bv^2}\vert_{\bv=0} \bv\right],\\
    \textbf{Var}_{q(\bv)}\left[\varphi_{\bx_0}(\bv)\right]&=\mathbb{E}_{q(\bv)}\left[ \frac{\partial \varphi_{\bx_0}(\bv)}{\partial\bv}\vert_{\bv=0}\mathbf{\Sigma}\frac{\partial \varphi_{\bx_0}(\bv)}{\partial\bv}\vert_{\bv=0}^\top\right]. 
\end{align*}}
This shows that parameter perturbations can be biased and their variance depends on $\mathbf{\Sigma}=\mathbf{H}_{\mathcal{L}}^{-1}\left(\btheta^\ast\right).$ 

\subsection{Why Perturbations Help Mitigate Memorisation}
\label{sub:pertubations-mitigate-memorisation}

We show that points that are on the edge of being memorised can gain sufficient distance from the training samples to be reclassified as not memorised. In the previous section we have shown that the perturbed parameters induce different vector fields. Hence, we expect that for some initial values we obtain different endpoints. Let $\bx_0$ be such an initial value, $\btheta^*$ the MAP estimate and $\btheta'$ the perturbed parameter. Suppose they induce the generators $g_{\btheta}$ and $g_{\btheta'}$ such that $$\widehat{\bx}'=g_{\btheta'}(\bx_0)\neq g_{\btheta}(\bx_0)=\widehat{\bx}.$$ We write $\boldsymbol{\delta}=\widehat{\bx}'-\widehat{\bx}$ to denote the displacement vector between the two generated points.  We will show that perturbations are most influential for data points that are close to the memorisation boundary: if a point $\widehat{\bx}$ is clearly memorised, or clearly not memorised, this is unlikely to change under perturbations. However, points that are only narrowly classified as memorised might get pushed away from the closest training point $\bx^{(1)}(\widehat{\bx})$ after a small perturbation of the vector field. This can be illustrated in a simplified case. 

\begin{theorem}\label{th:margin}
    Let $\widehat{\bx}$ be memorised, i.e. $\lVert \widehat{\bx} - \bx^{(1)}(\widehat{\bx})\rVert^2 \leq c \lVert \widehat{\bx} - \bx^{(2)}(\widehat{\bx})\rVert^2.$ Suppose that the perturbed point $\widehat{\bx}'=\widehat{\bx}+\boldsymbol{\delta}$ wanders away from $\bx^{(1)}(\widehat{\bx})$, and straight towards $\bx^{(2)}(\widehat{\bx})$. Then, $\widehat{\bx}'$ is not memorised only if 

\begin{align*}
    \lVert\boldsymbol{\delta}\rVert(1+\sqrt{c})&<\lVert \widehat{\bx} - \bx^{\left(2\right)}(\widehat{\bx})\rVert-\sqrt{c} \lVert \widehat{\bx} - \bx^{\left(1\right)}(\widehat{\bx})\rVert,\\
    \lVert\boldsymbol{\delta}\rVert(1+\sqrt{c})&>\sqrt{c} \lVert \widehat{\bx} - \bx^{\left(2\right)}(\widehat{\bx})\rVert-\lVert \widehat{\bx} - \bx^{\left(1\right)}(\widehat{\bx})\rVert.
\end{align*}
\end{theorem}
\begin{figure}[hbt!]
    \centering
   \vspace{-10pt}
    \includegraphics[width=0.7\linewidth]{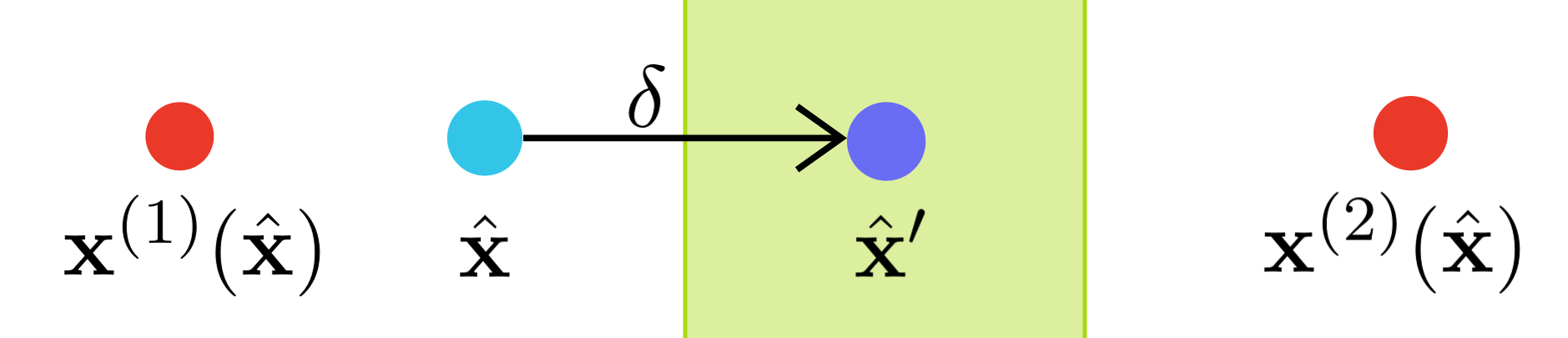}
   \vspace{5pt}
    \caption{In the setting of Theorem \ref{th:margin}, in which all considered points lie on one straight line, the perturbation makes a generated point (\turquoisedot) wander straight towards the second closest training sample (\reddot). If the new point (\periwinkledot) falls into the green zone, it is no longer classified as memorised. The size of the green zone depends on $\lVert \widehat{\bx} - \bx^{\left(1\right)}(\widehat{\bx})\rVert+\lVert \widehat{\bx} - \bx^{\left(1\right)}(\widehat{\bx})\rVert,$ and is maximised if $\widehat{\bx}$ is right between the two training samples. }
    \label{fig:perturbedpoint}
\end{figure}
For the proof, see Appendix \ref{app:theory}. 
Due to its strong assumptions, this argument is not intended to be quantitatively predictive, but it provides a useful qualitative explanation of why perturbations mitigate memorisation.

\begin{figure*}[tb]
    \centering

    \includegraphics[width=0.9\linewidth]{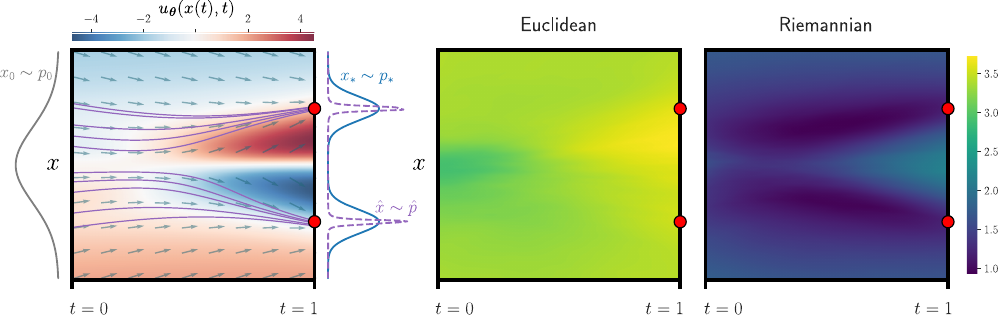}
    \caption{A 1D flow matching toy problem with a Gaussian base distribution $p_0$ and a GMM as the target distribution $p_\ast$. \textit{Left:} the learnt conditional vector field $u_{\btheta}\left(x(t), t\right):\rn\times[0,1]\rightarrow\rn$ at the optimal parameters $\btheta^\ast$ for $x(t) \in [-3, 3]$ and $t\in[0,1]$. Each line (\purpleline) is a trajectory $x(t)$ of a noise sample $x_0\sim p_0$ under $u_{\btheta^\ast}$, which yields $x(1)=\widehat{x}\sim p\left(\widehat{x}\right)$. This distribution overfits to the two fixed training points (\reddot), hence the generative model $g_{\btheta^\ast}$ has learnt to \textit{memorise} rather than generalise to $p_\ast$. \textit{Right:} we draw $S=1000$ models $\btheta^s \sim q\left(\btheta|\mathcal{D}\right)$ from the Euclidean and Riemannian approximate posterior. Each model $\btheta^s$ gives a specific velocity field (as in \textit{left}). We show the standard deviation per $\left(x(t),t\right)$-coordinate computed over the $S$ different velocity fields.}
    \label{fig:velocity-field}
\end{figure*}

\subsection{Analysis of the Riemannian Paths}
\label{sub:riemannian-paths}

Although perturbing the model's parameters can mitigate memorisation, it comes with the risk of deteriorating its generalisation capabilities. We show that samples from our Riemannian Laplace preserve generalisation better than its Euclidean counterpart, since the perturbed parameters stay in low loss regions. 
It is evident that the Euclidean perturbation upper bounds the Riemannian, as it moves in a straight line, which is the shortest path in the flat parameter space. The Riemannian sample follows a different path, which cannot lead further away than the straight line. 
\begin{theorem}
    Following the constructions in Subsection \ref{subsec:RLA}, we have that $|| \btheta^{s}_{\text{E}}-\btheta^\ast||\geq||\btheta^{s}_{\text{R}}-\btheta^\ast||$.
\end{theorem}
For the proof, see Appendix \ref{app:subsec:displacement}. 
Computation \ref{app:subsec:curve} in the Appendix approximates the curve $\balpha(t)$ using $n$ Euler steps of size $\epsilon=1/n$, and allows us to estimate $\balpha(1)$ for sufficiently small $\epsilon$.
\begin{equation}
    \btheta^{s}_{\text{R}}=\balpha(1)\approx\btheta^{s}_{\text{E}}+\epsilon^2\sum_{j=0}^{n-2}\left(n-1-j\right)\Ddot{\balpha}(j\epsilon).\label{eq:sum-approx}
\end{equation}
Combining  Equations \eqref{eq:geodesic-equation-monge}, \eqref{eq:taylor-u} and \eqref{eq:sum-approx}, we obtain 

\begin{equation}
    u_{\btheta^{s}_{\text{R}}}(\bx(t),t) \approx u_{\btheta^{s}_{\text{E}}}(\bx(t),t) -\epsilon^2 \boldsymbol{\kappa}^\top \nabla_\theta u_{\btheta^*}(\bx(t),t)\label{eq:riemann-geodesic}
\end{equation}
where the correction vector $\boldsymbol{\kappa}$ is given by
{\small
\begin{equation*}
    \boldsymbol{\kappa} = \sum_{j=0}^{n-2}(n-1-j)  \frac{\dot{\balpha}(j\epsilon)^\top \mathbf{H}_{\mathcal{L}}(\balpha(j\epsilon))\dot{\balpha}(j\epsilon)}{1+\lVert\nabla_\theta\mathcal{L}(\balpha(j\epsilon))\rVert^2}\nabla_\theta\mathcal{L}(\balpha(j\epsilon)).
\end{equation*}
}

This lets us compare how the Riemannian and Euclidean perturbations impact the vector field. The second term in Equation \eqref{eq:riemann-geodesic} opposes motion into directions of rapidly increasing loss. If $\dot{\balpha}(t)$ is aligned with directions where the loss increases sharply, then the correction term is of large magnitude. Specifically, if $\dot{\balpha}(t)$ points in the direction of the top-eigenvector of the Hessian, the difference between the straight (Euclidean) and geodesic (Riemannian) paths in the parameter space will be maximised. If however $\dot{\balpha}(t)$ points in directions where the loss in flat, the Euclidean and Riemannian vector fields will be close. Thus, the Riemannian method shrinks steps toward this direction, and helps us to stay in parameter regions with lower loss.

\section{Experiments}\label{sec:experiments}

\begin{figure*}[tb]
    \centering
    \includegraphics[width=0.75\linewidth]{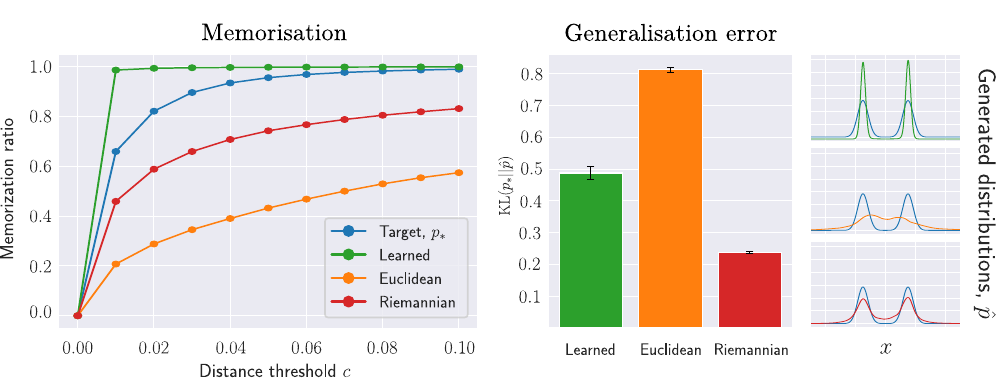}
    \caption{\textit{Left:} The memorisation ratio as a function of the distance threshold $c$ when generating data from the target distribution, the learnt distribution and the learnt distribution under the two Bayesian treatments. \textit{Right:} The generalisation error computed as the $\operatorname{KL}$-divergence between the target distribution and the generated data distributions. 
    For efficiency, we perform $50$ repetitions of computing $\operatorname{KL}$-divergences from a subset of $100$ generated data samples and plot the means and standard errors in the bar plot.
    The generated data distributions (\textit{most right}) correspond to pushing noise samples from $p_0$ through the generator $g_{\btheta}$ using the learnt model $\btheta^\ast$ (top), using several models sampled from the Euclidean Laplace approximation (middle), and using several models sampled from the Riemannian Laplace approximation (bottom). We visualise the resulting distributions using kernel density estimation. See Appendix \ref{app:experiments} for details.}
    \label{fig:results}
\end{figure*}

\subsection{Toy Experiments}

We first consider the problem of learning a flow matching model in $D$ dimensions with the base distribution being an isotropic Gaussian, $p_0\sim\mathcal{N}\left(\boldsymbol{0}, \mathbb{I}_D\right)$, and the target distribution $p_\ast$ being a Gaussian mixture model (GMM). See Appendix \ref{app:experiments} for details on the setup, and results for $D=2$.

\paragraph{1D:} For $D=1$, we define the target distribution to have means $\{\mu_1, \mu_2\}=\{-1.5, 1.5\}$ and variances $\sigma_1^2=\sigma_2^2=0.1$. For learning a generative model that \textit{memorises} (i.e. overfits), we construct a naively simple training set that consists of only $2$ points: $\mu_1$ and $\mu_2$. We show the learning problem and the overfitted velocity field in Figure \ref{fig:velocity-field} (\textit{left}).

We adopt a Riemannian Bayesian treatment of the model parameters following Section \ref{sec:bayesian}. We then define the distributions of generated data samples under the two approximate posteriors as in Equation \eqref{eq:posterior-predictive} by sampling $S=1000$ samples from the posterior $\btheta^s \sim q(\btheta|\mathcal{D}_{\bx})$ and $N=1000$ samples from the base distribution $\bx_{0} \sim p_0$. Note that each posterior sample induces a different velocity field. 

Figure \ref{fig:velocity-field} (\textit{right}) shows the uncertainty at each $(x_t, t)$-coordinate, computed as the standard deviation over the $S$ posterior samples. Notably, the velocity fields induced by the Riemannian samples are far less uncertain than the Euclidean ones, especially along the path where the original vector field has low magnitude. This is also seen in the symmetry patterns: the learnt vector field is an odd function in $x,$ while the Riemannian uncertainty is even, suggesting that velocity fields associated to perturbations from the Riemannian posterior still respect the symmetry of the original model. The Euclidean heat map has no such symmetry. 

We provide the memorisation and generalisation capabilities in Figure \ref{fig:results} and see that perturbing the model parameters reduces memorisation, no matter the choice of distance threshold $c$. This supports our argument in Subsection \ref{sub:pertubations-mitigate-memorisation}. Although the generated samples using Euclidean perturbations exhibit less memorisation than the Riemannian ones, this comes at the cost of poorer generalisation, measured as the alignment with the underlying target distribution.

\subsection{CIFAR-10}

We also apply our strategy to an unconditional diffusion model based on a U-Net architecture with roughly 36B parameters. As argued by \citet{gu2023memorization}, training on large amounts of training data results in limited to no memorisation. Therefore, we train a diffusion model on a subset of $N=1000$ training samples from CIFAR-10 \cite{krizhevsky2009learning}, yielding the MAP $\btheta^\ast$. For computational reasons, we restrict our Bayesian treatment to be on a subset of the parameters, specifically the parameters of the first layer.
Note that this construction serves as a proof-of-concept for applying our technique to high-dimensional settings, and that this approximate posterior does not capture the full flexibility of our approach.

Next, we approximate the Hessian at the MAP, noting that its spectrum is dominated by a few large eigenvalues, with most others being near zero (as also observed by \citet{roy2024reparameterization}). This indicates that only certain directions in the parameter space strongly influence the loss, and that it remains relatively flat along most other directions. Consequently, Riemannian perturbations differ from Euclidean ones mainly along the top-eigenvectors directions.

\begin{figure}[tb!]
    \includegraphics[width=0.97\linewidth]{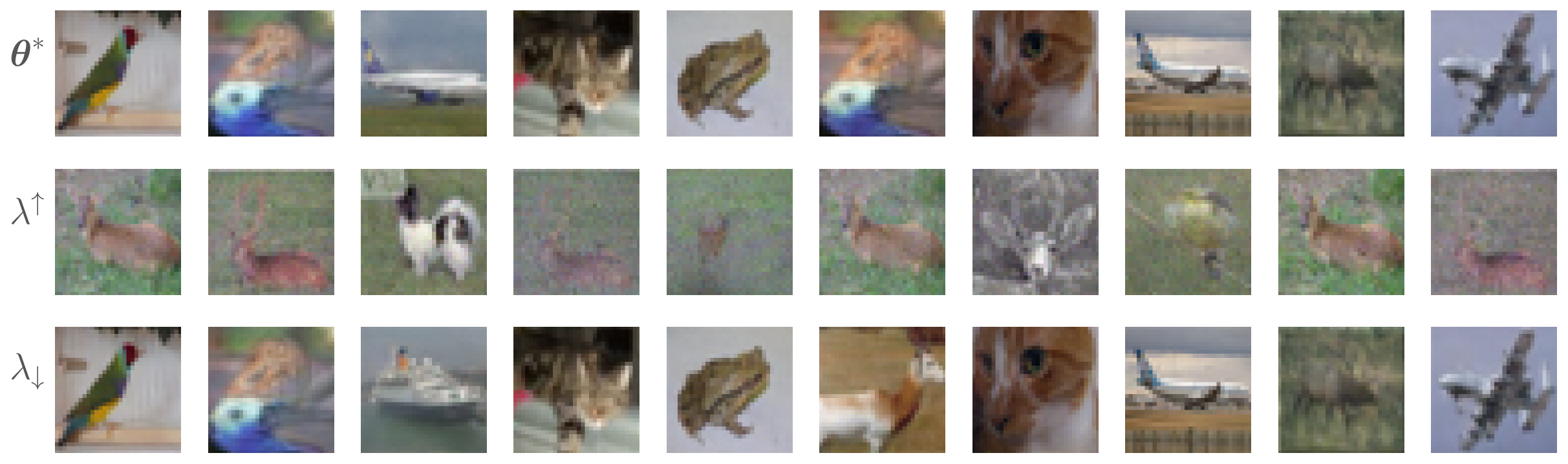}
    \hfill
    \caption{Generated images with the MAP model \textit{(top)} and under Riemannian perturbations along the eigenvectors with highest and lowest positive eigenvalues (\textit{middle} and \textit{bottom} row). For each column, we use a fixed initial value $\bx_0\sim p_0$. Along the bottom-eigenvector, the parameter space is flatter, and the generated images are close to the MAP generator's images. Images generated under perturbation by the top-eigenvector tend to converge to a few specific images: for three distinct $\bx_0,$ we obtain the same deer.}
    \label{fig:eigenvector-samples-endpoints}
\end{figure}
\begin{figure}[tb!]
    \vspace{-5pt}
    \includegraphics[width=0.97\linewidth]{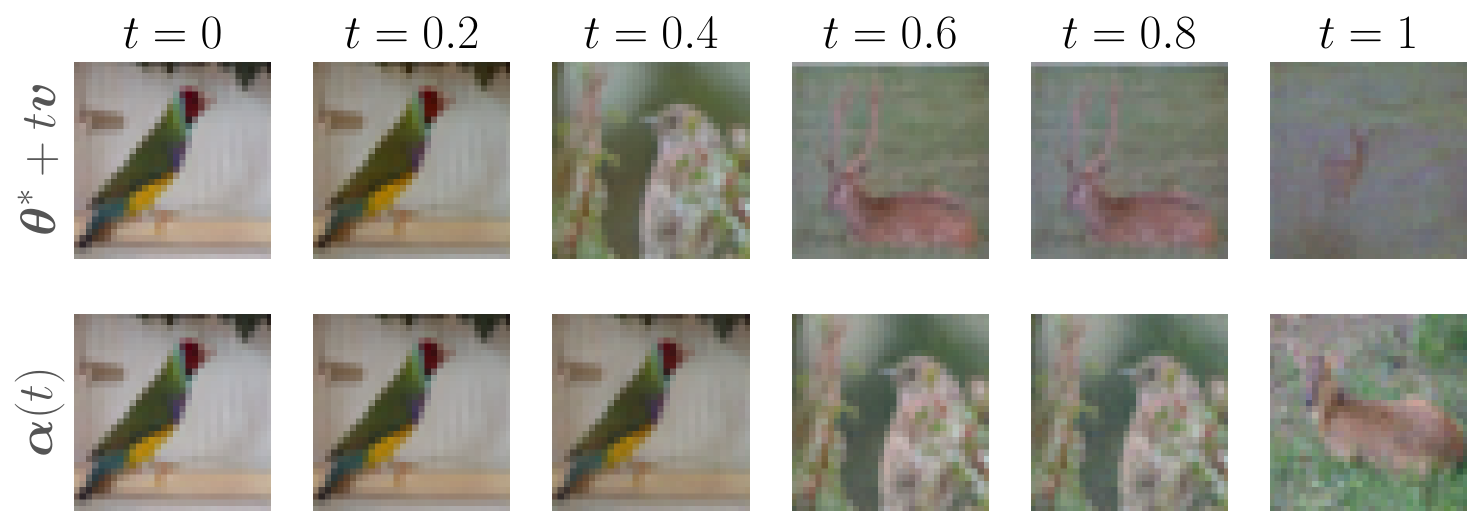}
    \hfill
    \caption{How a sample from a generative model changes when gradually perturbing the parameters along the path induced by the top-eigenvector. We fix the noise sample $\bx_0$ for all images. \textit{Top:} along the straight path $\btheta^\ast + t\bv$. \textit{Bottom:} along the curve $\balpha(t)$.}
    \vspace{-9pt}
    \label{fig:largest-eigenvector-compare-trajectories}
\end{figure}

We conduct a structured experiment by sampling along the most curved direction as well as the direction of the lowest positive eigenvalue, rather than using the full inverse Hessian as the Laplace covariance. In Figure \ref{fig:eigenvector-samples-endpoints}, we show the generated images when following the Riemannian perturbation technique with these eigenvectors as the initial velocities. While images generated using the bottom-eigenvector perturbation are identical or qualitatively close to the images generated with the MAP model, images generated with the top-eigenvector are visually very different, yet result in a similarly looking image for all considered base distribution samples. This indicates that the resulting vector field is a sink: there exists one dominant point $\widehat{\bx},$ such that all nearby paths converge towards it. Understanding this behaviour and how to avoid it constitutes future work.

In Figure \ref{fig:largest-eigenvector-compare-trajectories}, we experimentally verify the claim from Subsection \ref{sub:riemannian-paths}, namely that the Riemannian paths stay closer to well-generalising generators. By again considering the worst-case perturbation along the top-eigenvector, we see that the Euclidean approach yields a generator that ends up in a parameter space location where meaningless images are generated. As expected, the Riemannian paths favour low-loss regions, conserving a high image quality.

Although our method is not optimised for the image domain, we show that the parameters of complex generators can be manipulated in a controlled way. Further implications and connections to memorisation are discussed in Appendix \ref{app:image-experiments}.

\section{Related Work and Discussion}

\subsection{What Drives Generalisation?}
At present, it remains unresolved whether diffusion models truly generalise at all, and if so, how and why. \citet{kadkhodaie2023generalization} demonstrate empirically that large datasets are necessary to generalise, as they allow diffusion models to recover a geometry-aligned score function. Another promising line of work reveals that the amount of memorisation in generative models increases with training time and is inversely proportional to the size of the training set \cite{gu2023memorization, biroli2024dynamical, bonnaire2025diffusion}. In contrast, our work focuses on reducing memorisation after the generative model is already fully trained. 
We do so by investigating how to represent uncertainty and preserve generalisation around the learnt model, independent of further increases in data and training time.

\citet{bertrand2025closed} demonstrate that target stochasticity (the randomness of the training target arising from the sampling process) is not driving generalisation – in fact, there exists a closed form solution for the flow matching loss. Interestingly, this closed form solution only reproduces training samples. They conclude that generalisation occurs precisely when the network \emph{fails} to fit the optimal vector field. Our work, however, introduces stochasticity in the parameter space. By perturbing the optimal vector field, we \emph{deliberately prevent} the collapse to a deterministic flow, and achieve the described generalisation inducing effect.

\subsection{Balancing Generalisation and Memorisation.}
Our results highlight that reducing memorisation alone does not guarantee improved generalisation; rather, we observe a delicate trade-off between memorisation and generalisation. The Euclidean method exhibits the least memorisation, since the parameter perturbations encourage a broader variance over the output space. However, this comes at the cost of poorer generalisation, demonstrating that counteracting memorisation too aggressively can prevent the model from learning meaningful structures. The Riemannian method, on the other hand, provides a good balance: it still exhibits less memorisation than the baseline, while achieving substantially stronger generalisation.

\subsection{How Should We Measure Memorisation?}\label{subsec:memorisationmeasure}
The nearest neighbours memorisation ratio defined in Equation \eqref{eq:memorisation-ratio} is easy to evaluate, however a downside is the dependence on Euclidean distance, which does not take the structure of the data manifold into account and therefore does not necessarily align with semantic similarity in the data domain. In the case of image data, previous work \citep{pizzi2022self, chen2024towards} considers embedding the generated images and training images to a semantic latent space using a pre-trained encoder, e.g. trained using self-supervised learning. A different line of research \cite{khrulkov2020hyperbolic, atigh2022hyperbolic} embeds image data into the hyperbolic space instead. Hyperbolic manifolds have a negative curvature and are thus better at representing hierarchical data with minimal distortion. In the image domain, visual concepts can induce a hierarchy (for example, felines and canines have distinct shapes), but we can also consider a hierarchy based on information content or degradation.

Further, the constant $c$ is somewhat arbitrary: in the image domain, for example, it is often set to $c=\frac{1}{3}$ since this value has been found to work well empirically on the considered data sets. In the future, we hope to see a theory-backed rule of choosing an appropriate $c$ for a given data set in a way that aligns with human perception. Lastly, the nearest neighbour memorisation measure suffers from the curse of dimensionality: pairwise distances between samples tend to concentrate as the dimension increases, limiting the interpretability of the measure. We also want to remark that this criterion becomes meaningless in case of duplicate training points, which is the case for CIFAR-10.

Alternative measures for memorisation have been proposed, for example by \citet{ross2024geometric}. Their work compares the local intrinsic dimension (LID) of the ground truth data manifold and of the learnt data manifold. They label a point as memorised if the dimension of the learnt manifold is lower than the dimension of the ground truth manifold. This measure carries a nice geometric flavour and provides more qualitative information about the structure of the learnt distribution than merely evaluating distances between points.
However, the estimation of the LID is computationally heavy and harder to evaluate, although improvements have been made by \citet{kamkari2024geometricviewdatacomplexity}. We expect that our method also helps mitigate this notion of memorisation: if a perturbed vector field leads to a different endpoint, this increases the dimension of the learnt manifold.

\section{Conclusion} 
Our work presents a Bayesian and Riemannian framework for exploring the parameter space of large generative models. Adopting a geometry-informed Bayesian treatment of the model parameters can help reduce memorisation, without loosing generalisation capabilities. This finding is based both on theoretical and empirical foundations: Section \ref{sec:theory} provides a mathematical explanation, while Section \ref{sec:experiments} contains a one-dimensional example, allowing for visually interpreting the effects of our Riemannian Bayesian treatment of the model parameters. We lastly showed that our method scales to settings with high-dimensional parameter spaces.

\subsection{Limitations and Future Work}

Rather than competing with state-of-the-art image generators, we demonstrated the feasibility of our approach. Although our approximate posterior is flexible, it is inherently mode-based and centred around the MAP solution. In the image domain, the generated data points closely resemble the training data with added unstructured noise, which ideally should be structured to change the image concept. Furthermore, due to computational constraints in large models, we restrict the method to a subset of the parameter space. We expect that specialised geodesic solvers for high-dimensional manifolds could lead to improvements.

\section*{Acknowledgments}
This work was supported by the Danish Data Science Academy, which is funded by the Novo Nordisk Foundation (NNF21SA0069429) and VILLUM FONDEN (40516), and by the DFF Sapere Aude Starting Grant ``GADL'' and by the Horizon-EU EIC Pathfinder Open program to i-RASE: Intelligent Radiation Sensor Readout System, with reference number: HE EIC - i-RASE – 101130550. We thank Stas Syrota and Federico Bergamin for early discussions that helped shaping our understanding of the relationship between diffusion and flow matching.

\section{Impact Statement}
Generative models are increasingly deployed in high-stakes applications, and thus, mitigating memorisation is crucial. We propose a novel perspective on tackling this issue, and as a result, the proposed approach can help mitigate risks associated with memorisation, including the reproduction of sensitive data. The general aim of reducing memorisation in generative models can seem as appealing in order to avoid copyright infringement issues. However, as the underlying process becomes less transparent, it becomes harder for e.g. independent artists to prove such violations.

Beyond these considerations, our work contributes to the broader goal of building more reliable and robust generative models, and better understanding their underlying mechanisms.

\bibliographystyle{plainnat} 
\bibliography{refs}

@article{hartmann2022lagrangian,
  title={{Lagrangian manifold Monte Carlo on Monge patches}},
  author={Hartmann, Marcelo and Girolami, Mark and Klami, Arto},
  journal={International Conference on Artificial Intelligence and Statistics},
  year={2022}
}

@article{williams2025geodesic,
  title={{Geodesic slice sampler for multimodal distributions with strong curvature}},
  author={Williams, Bernardo and Yu, Hanlin and Luu, Hoang Phuc Hau and Arvanitidis, Georgios and Klami, Arto},
  journal={Uncertainty in Artificial Intelligence (UAI)},
  year={2025}
}

@article{jacobsen2025monge,
  title={{Monge SAM: Robust reparameterization-invariant sharpness-aware minimization based on loss geometry}},
  author={Jacobsen, Albert Kj{\o}ller and Arvanitidis, Georgios},
  journal={arXiv:2502.08448},
  year={2025}
}

@article{bergamin2023riemannian,
  title={{Riemannian Laplace approximations for Bayesian neural networks}},
  author={Bergamin, Federico and Moreno-Mu{\~n}oz, Pablo and Hauberg, S{\o}ren and Arvanitidis, Georgios},
  journal={Neural Information Processing Systems (NeurIPS)},
  year={2023}
}

@article{yu2023riemannian,
  title={{Riemannian Laplace approximation with the Fisher metric}},
  author={Yu, Hanlin and Hartmann, Marcelo and Williams, Bernardo and Girolami, Mark and Klami, Arto},
  journal={International Conference on Artificial Intelligence and Statistics (AISTATS)},
  year={2024}
}

@article{ross2024geometric,
  title={{A geometric framework for understanding memorization in generative models}},
  author={Ross, Brendan Leigh and Kamkari, Hamidreza and Wu, Tongzi and Hosseinzadeh, Rasa and Liu, Zhaoyan and Stein, George and Cresswell, Jesse C and Loaiza-Ganem, Gabriel},
  journal={Neural Information Processing Systems (NeurIPS)},
  year={2024}
}

@article{buchanan2025edge,
  title={{On the edge of memorization in diffusion models}},
  author={Buchanan, Sam and Pai, Druv and Ma, Yi and De Bortoli, Valentin},
  journal={Neural Information Processing Systems (NeurIPS)},
  year={2025}
}

@article{bertrand2025closed,
  title={On the Closed-Form of Flow Matching: Generalization Does Not Arise from Target Stochasticity},
  author={Bertrand, Quentin and Gagneux, Anne and Massias, Mathurin and Emonet, R{\'e}mi},
  journal={Neural Information Processing Systems (NeurIPS)},
  year={2025}
}

@article{jacobsen2025staying,
  title={Staying on the manifold: Geometry-Aware Noise Injection},
  author={Jacobsen, Albert Kj{\o}ller and Gegenfurtner, Johanna Marie and Arvanitidis, Georgios},
  journal={Northern Lights Deep Learning Conference (NLDL)},
  year={2026}
}

@article{jazbec2025generative,
  title={Generative Uncertainty in Diffusion Models},
  author={Jazbec, Metod and Wong-Toi, Eliot and Xia, Guoxuan and Zhang, Dan and Nalisnick, Eric and Mandt, Stephan},
  journal={Uncertainty in Artificial Intelligence (UAI)},
  year={2025}
}

@book{Lee00,
  author = {Lee, John M.},
  
  title = {Introduction to Smooth Manifolds},
publisher = {Springer New York},
  year = 2000
}

@article{sohldickstein2015deepunsupervisedlearningusing,
      title={Deep Unsupervised Learning using Nonequilibrium Thermodynamics}, 
      author={Jascha Sohl-Dickstein and Eric A. Weiss and Niru Maheswaranathan and Surya Ganguli},
        journal={arXiv:1503.03585},
year = 2015
}

@misc{flowsanddiffusions2025,
    author       = {Peter Holderrieth and Ezra Erives},
    title        = {Introduction to Flow Matching and Diffusion Models},
    year         = {2025},
    url          = {https://diffusion.csail.mit.edu/}
  }

@article{ho2020denoising,
  title={Denoising diffusion probabilistic models},
  author={Ho, Jonathan and Jain, Ajay and Abbeel, Pieter},
  journal={Neural Information Processing Systems (NeurIPS)},
  year={2020}
}

@book{Gudmundsson2025_Riemann,
  author       = {Sigmundur Gudmundsson},
  title        = {An Introduction to Riemannian Geometry},
  institution  = {Lund University},
  type         = {Lecture Notes in Mathematics},
  version      = {1.433 (24 April 2025)},
  year         = {2025},
  url          = {https://www.matematik.lu.se/matematiklu/personal/sigma/Riemann.pdf},
  }

@article{kamkari2024geometricviewdatacomplexity,
      title={A Geometric View of Data Complexity: Efficient Local Intrinsic Dimension Estimation with Diffusion Models}, 
      author={Hamidreza Kamkari and Brendan Leigh Ross and Rasa Hosseinzadeh and Jesse C. Cresswell and Gabriel Loaiza-Ganem},
      year={2024},
      journal={arxiv:2406.03537}
}

@inproceedings{yoon2023diffusion,
  title={Diffusion probabilistic models generalize when they fail to memorize},
  author={Yoon, TaeHo and Choi, Joo Young and Kwon, Sehyun and Ryu, Ernest K},
  booktitle={ICML 2023 workshop on structured probabilistic inference \& generative modeling},
  year={2023}
}

@article{dormand1980family,
  title={{A family of embedded Runge-Kutta formulae}},
  author={Dormand, John R. and Prince, Peter J.},
  journal={Journal of computational and applied mathematics},
  volume={6},
  number={1},
  pages={19--26},
  year={1980},
  publisher={Elsevier}
}

@article{biroli2024dynamical,
  title={Dynamical regimes of diffusion models},
  author={Biroli, Giulio and Bonnaire, Tony and De Bortoli, Valentin and M{\'e}zard, Marc},
  journal={Nature Communications},
  volume={15},
  number={1},
  pages={9957},
  year={2024},
  publisher={Nature Publishing Group UK London}
}

@article{smith2021origin,
  title={On the origin of implicit regularization in stochastic gradient descent},
  author={Smith, Samuel L and Dherin, Benoit and Barrett, David GT and De, Soham},
  journal={arXiv preprint arXiv:2101.12176},
  year={2021}
}

@article{song2021maximum,
  title={Maximum likelihood training of score-based diffusion models},
  author={Song, Yang and Durkan, Conor and Murray, Iain and Ermon, Stefano},
  journal={Advances in neural information processing systems},
  volume={34},
  pages={1415--1428},
  year={2021}
}

@inproceedings{gao2025diffusionmeetsflow,
  author = {Gao, Ruiqi and Hoogeboom, Emiel and Heek, Jonathan and Bortoli, Valentin De and Murphy, Kevin P. and Salimans, Tim},
  title = {Diffusion Meets Flow Matching: Two Sides of the Same Coin},
  year = {2024},
  url  = {https://diffusionflow.github.io/}
}

@inproceedings{khrulkov2020hyperbolic,
  title={Hyperbolic image embeddings},
  author={Khrulkov, Valentin and Mirvakhabova, Leyla and Ustinova, Evgeniya and Oseledets, Ivan and Lempitsky, Victor},
  booktitle={Proceedings of the IEEE/CVF conference on computer vision and pattern recognition},
  pages={6418--6428},
  year={2020}
}

@inproceedings{atigh2022hyperbolic,
  title={Hyperbolic image segmentation},
  author={Atigh, Mina Ghadimi and Schoep, Julian and Acar, Erman and Van Noord, Nanne and Mettes, Pascal},
  booktitle={Proceedings of the IEEE/CVF conference on computer vision and pattern recognition},
  pages={4453--4462},
  year={2022}
}

@article{gu2023memorization,
  title={On Memorization in Diffusion Models},
  author={Gu, Xiangming and Du, Chao and Pang, Tianyu and Li, Chongxuan and Lin, Min and Wang, Ye},
  journal={Transactions on Machine Learning Research (TMLR)},
  year={2025},
}

@article{kadkhodaie2023generalization,
  title={Generalization in diffusion models arises from geometry-adaptive harmonic representations},
  author={Kadkhodaie, Zahra and Guth, Florentin and Simoncelli, Eero P and Mallat, St{\'e}phane},
  journal={arXiv preprint arXiv:2310.02557},
  year={2023}
}

@inproceedings{chen2024towards,
  title={Towards memorization-free diffusion models},
  author={Chen, Chen and Liu, Daochang and Xu, Chang},
  booktitle={Proceedings of the IEEE/CVF Conference on Computer Vision and Pattern Recognition},
  pages={8425--8434},
  year={2024}
}

@article{bonnaire2025diffusion,
  title={Why Diffusion Models Don't Memorize: The Role of Implicit Dynamical Regularization in Training},
  author={Bonnaire, Tony and Urfin, Rapha{\"e}l and Biroli, Giulio and M{\'e}zard, Marc},
  journal={Neural Information Processing Systems (NeurIPS)},
  year={2025}
}

@inproceedings{pizzi2022self,
  title={A self-supervised descriptor for image copy detection},
  author={Pizzi, Ed and Roy, Sreya Dutta and Ravindra, Sugosh Nagavara and Goyal, Priya and Douze, Matthijs},
  booktitle={Proceedings of the IEEE/CVF Conference on Computer Vision and Pattern Recognition},
  pages={14532--14542},
  year={2022}
}

@article{krizhevsky2009learning,
  title={Learning multiple layers of features from tiny images},
  author={Krizhevsky, Alex and Hinton, Geoffrey and others}, year={2009},
url = {https://www.cs.utoronto.ca/~kriz/learning-features-2009-TR.pdf}


}

@article{roy2024reparameterization,
  title={Reparameterization invariance in approximate Bayesian inference},
  author={Roy, Hrittik and Miani, Marco and Ek, Carl Henrik and Hennig, Philipp and Pf{\"o}rtner, Marvin and Tatzel, Lukas and Hauberg, S{\o}ren},
  journal={Advances in Neural Information Processing Systems},
  volume={37},
  pages={8132--8164},
  year={2024}
}

@article{liu2024generative,
  title={Generative AI model privacy: a survey},
  author={Liu, Yihao and Huang, Jinhe and Li, Yanjie and Wang, Dong and Xiao, Bin},
  journal={Artificial Intelligence Review},
  volume={58},
  number={1},
  pages={33},
  year={2024},
  publisher={Springer}
}

@inproceedings{songdenoising,
  title={Denoising Diffusion Implicit Models},
  author={Song, Jiaming and Meng, Chenlin and Ermon, Stefano},
  booktitle={International Conference on Learning Representations (ICLR)},
  year={2021}
}

@article{heusel2017gans,
  title={Gans trained by a two time-scale update rule converge to a local nash equilibrium},
  author={Heusel, Martin and Ramsauer, Hubert and Unterthiner, Thomas and Nessler, Bernhard and Hochreiter, Sepp},
  journal={Advances in neural information processing systems},
  volume={30},
  year={2017}
}

@article{lanczos1950iteration,
  title={An iteration method for the solution of the eigenvalue problem of linear differential and integral operators},
  author={Lanczos, Cornelius},
  journal={Journal of research of the National Bureau of Standards},
  volume={45},
  number={4},
  pages={255--282},
  year={1950}
}

@book{villani2008optimal,
  title={Optimal transport: old and new},
  author={Villani, C{\'e}dric and others},
  volume={338},
  year={2008},
  publisher={Springer}
}

@article{peyre2019computational,
  title={Computational optimal transport: With applications to data science},
  author={Peyr{\'e}, Gabriel and Cuturi, Marco and others},
  journal={Foundations and Trends{\textregistered} in Machine Learning},
  volume={11},
  number={5-6},
  pages={355--607},
  year={2019},
  publisher={Now Publishers, Inc.}
}

\newpage
\appendix
\onecolumn
\section{Essential Facts on Geodesics}\label{app:geodesics}
In this section, we will explain a few concepts we use when constructing the Riemannian Laplace approximation \ref{subsec:RLA}.

The Christoffel symbols describe how much basis vectors evolve along a manifold, and are useful for computations, including the geodesic equation \eqref{eq:geodesic-equation}.
We have the following formal definition.
\begin{definition}\cite{Gudmundsson2025_Riemann}
    Let $(\mathcal{M},g)$ be a Riemannian manifold with Levi-Civita connection $\nabla.$ Further let $(U,x)$ be a local chart on $\mathcal{M}$ and put $X_i=\frac{\partial}{\partial x_i}\in\mathcal{C}^\infty(TU),$ so that $\{X_1,\dots,X_m\}$ is a local frame for $T\mathcal{M}$ on $U.$ Then we define the Christoffel symbols $\Gamma_{ij}^k:U\rightarrow\rn$ of the connection $\nabla$ with respect to $(U,x)$ by 
    $$\nabla_{X_i}X_j=\sum_{k=1}^m\Gamma_{ij}^k\cdot X_k.$$
     Let $\mathbf{G}$ denote the matrix form of the metric $g,$ that is $\mathbf{G}_{ij}=g(X_i,X_j)$. Then, $\Gamma_{ij}^k$ is given by
    $$\Gamma_{ij}^k=\frac{1}{2}\sum_{l=1}^m\left(\mathbf{G}^{-1}\right)_{kl}\left\{\frac{\partial \mathbf{G}_{jl}}{\partial x_i}+\frac{\partial \mathbf{G}_{li}}{\partial x_j}-\frac{\partial \mathbf{G}_{ij}}{\partial x_l}\right\}.$$
   \end{definition}

This formula lets us compute the Christoffel symbols easily.

We also mention the exponential map. It essentially wraps the tangent space at a given point around the manifold. In two dimensions this is easy to picture: take an apple or a bagel, whose surfaces can be approximated by a sphere or a torus, respectively. If you stick aluminium foil to one particular point, you will be able to smoothly fit a small patch of the foil around the point without it wrinkling or tearing. Similarly, the exponential map takes vectors from a subset of the tangent space at the chosen point and maps them to the manifold, specifically by evaluating the geodesics with the given initial velocity.
\begin{definition}\cite{Gudmundsson2025_Riemann}
Let $(\mathcal{M},g)$ be an $m$-dimensional Riemannian manifold, $p\in\mathcal{M}$ and 
$$S_p^{m-1}=\{\bv\in T_p\mathcal{M} \ | \ g_p(\bv,\bv)=1\}.$$ For every $\bv\in S_p^{m-1}$ let $\bgamma_\bv:(-a_\bv,b_\bv)\rightarrow\mathcal{M}$ be the maximal geodesic such that $a_\bv,b_\bv\in\rn^+,$ $\bgamma(0)=p,$ and $\dot{\bgamma}(0)=\bv.$
It can be shown that $\epsilon_p=\inf\{a_\bv,b_\bv \ | \ \bv\in S_p^{m-1}\}$ is positive, and thus the open ball $$\mathcal{B}_{\epsilon_p}^m(0)=\{\bv\in T_p\mathcal{M} \ | \ g_p(\bv,\bv)<\epsilon_p\}$$ is non-empty.
The exponential map $\operatorname{Exp}_p:\mathcal{B}_{\epsilon_p}^m(0)\rightarrow\mathcal{M}$ at $p$ is given by
\begin{eqnarray}
    \operatorname{Exp}_{p}(\boldsymbol{v})=\begin{cases}
        \gamma_{\frac{\boldsymbol{v}}{\lVert\boldsymbol{v}\rVert}}\left(\lVert\boldsymbol{v}\rVert\right)& \text{if} \ \  \boldsymbol{v}\in \mathcal{B}_{\epsilon_p}^m(0)\backslash \{\boldsymbol{0}\},\\
        p & \text{if} \ \  \boldsymbol{v}=\boldsymbol{0}.
    \end{cases}
\end{eqnarray}
\end{definition}

\section{Details from the Theory Section}\label{app:theory}
\subsection{Proof of Theorem \ref{th:margin}}
\begin{proof}\label{pr:margin}
We assume that $\widehat{\bx}$ and $\bx'=\widehat{\bx}+\boldsymbol{\delta}$ lie on the straight line between $\bx^{(1)}(\widehat{\bx})$ and $\bx^{\left(2\right)}(\widehat{\bx}),$ such that  $\widehat{\bx}$ is closer to $\bx^{(1)}(\widehat{\bx})$ than $\bx'$ is.
Hence, we have that 
    \begin{eqnarray}\label{eq:distance-split}
        \lVert\widehat{\bx'}-\bx^{\left(1\right)}(\widehat{\bx})\rVert&=&\lVert\widehat{\bx}-\bx^{\left(1\right)}(\widehat{\bx})\rVert+\lVert\boldsymbol{\delta}\rVert,\\
        \lVert\widehat{\bx'}-\bx^{\left(2\right)}(\widehat{\bx})\rVert&=&\lVert\widehat{\bx}-\bx^{\left(2\right)}(\widehat{\bx})\rVert-\lVert\boldsymbol{\delta}\rVert.\label{eq:distances2}
    \end{eqnarray}
For $\widehat{\bx}'$ to not be memorised, it cannot be too close to either $\bx^{\left(1\right)}(\widehat{\bx})$ or $\bx^{\left(2\right)}(\widehat{\bx}).$ Specifically, we need
\begin{eqnarray}\label{eq:necessary}
    \lVert \widehat{\bx}' - \bx^{\left(1\right)}(\widehat{\bx})\rVert &>& \sqrt{c} \lVert \widehat{\bx}' - \bx^{\left(2\right)}(\widehat{\bx})\rVert,\\
    \lVert \widehat{\bx}' - \bx^{\left(2\right)}(\widehat{\bx})\rVert &>& \sqrt{c} \lVert \widehat{\bx}' - \bx^{\left(1\right)}(\widehat{\bx})\rVert^2.\label{eq:necessary2}
\end{eqnarray}
Plugging now Equations \eqref{eq:distance-split} and \eqref{eq:distances2} into \eqref{eq:necessary} and \eqref{eq:necessary2}, and moving terms around, we obtain the desired bound.
\end{proof}

\subsection{Comparing Displacement Under the Riemannian and Euclidean Perturbation.}\label{app:subsec:displacement}
We will provide a standard argument which shows that $||\btheta^{s}_{\text{E}}-\btheta^\ast||\geq||\btheta^{s}_{\text{R}}-\btheta^\ast||.$
\begin{proof}
    Recall that geodesics have unit speed, hence $$\lVert\dot{\bgamma}(t)\rVert\equiv\lVert\dot{\bgamma}(0)\rVert=\lVert\mathbf{J}_h(\btheta^*)\bv\rVert.$$ Since $\btheta^*$ minimises $\mathcal{L}(\btheta),$ we have that $\mathbf{J}_h(\btheta^*)=\begin{bmatrix}\mathbb{I}_K & \boldsymbol{0}\end{bmatrix}^\top.$ Thus, $\lVert\dot{\bgamma}(t)\rVert\equiv \lVert\bv\rVert.$
It now follows that $\bgamma(t)_{[0,1]}$ has arc length $\lVert\bv\rVert$:
$$\int_0^1\lVert\dot{\bgamma}(t)\rVert\text{d}t=\int_0^1\lVert\bv\rVert\text{d}t=\lVert\bv\rVert.$$
Using the definition of $\mathbf{G}(\btheta)$ from Equation \eqref{eq:monge-metric}, we now observe that 
$$\lVert\dot{\balpha}(t)\rVert_{\mathbf{G}(\btheta)}^2\equiv\lVert\dot{\balpha}(t)\rVert^2+\underset{\geq0}{\underbrace{\dot{\balpha}(t)^\top\nabla_\btheta\mathcal{L}(\btheta)\nabla_\btheta\mathcal{L}(\btheta)^\top\dot{\balpha}(t)}},$$ and consequently $\lVert\dot{\balpha}(t)\rVert\leq\lVert\dot{\balpha}(t)\rVert_{\mathbf{G}(\btheta)}. $
We now evaluate the arc length of $\balpha(t)_{[0,1]}$:
$$\int_0^1\lVert\dot{\balpha}(t)\rVert\text{d}t\leq\int_0^1\lVert\dot{\balpha}(t)\rVert_{\mathbf{G}(\btheta)}\text{d}t=\int_0^1\lVert\dot{\bgamma}(t)\rVert\text{d}t=\lVert\bv\rVert,$$ with equality if and only if $\balpha(t)$ is a straight line. As straight lines are the shortest paths in the flat Euclidean space, $\balpha(1)=\btheta^{s}_{\text{R}}$ cannot be further away from $\btheta^\ast$ than $\btheta^\ast+\bv=\btheta^{s}_{\text{E}}.$
\end{proof}

\subsection{Approximation of the Curve}\label{app:subsec:curve}
Assume $\balpha(t):[0,1]\subseteq\rn\rightarrow\Theta$ is a smooth curve in $\Theta.$ We wish to approximate $\balpha(1).$

We choose $n\in\mathbb{N}_{>0},$ and take discrete time steps of size $\epsilon=\frac{1}{n}.$ At each point in time $t,$ we take a step in the direction of $\dot{\balpha}(t),$ and thus
\begin{equation}\label{eq:approx}
    \balpha(t+\epsilon)\approx\balpha(t)+\epsilon\dot{\balpha}(t).
\end{equation}

After $n$ such steps we obtain
\begin{eqnarray*}
    \balpha(1)&\approx&\balpha(0)+\epsilon\dot{\balpha}(0)+\epsilon\dot{\balpha}(\epsilon)+\epsilon\dot{\balpha}(2\epsilon)+\dots+\epsilon\dot{\balpha}((n-1)\epsilon)\\
    &=&\balpha(0)+\epsilon\cdot\sum_{j=0}^{n-1}\dot{\balpha}(j\epsilon).
\end{eqnarray*}
Similar to Equation \eqref{eq:approx}, we can update the second derivatives by letting 
\begin{equation}\label{eq:approx-first}
    \dot{\balpha}(t+\epsilon)\approx\dot{\balpha}(t)+\epsilon\Ddot{\balpha}(t).
\end{equation}

This yields
\begin{eqnarray*}
    \balpha(1)
    &\approx&
    \balpha(0)+\epsilon\dot{\balpha}(0)+\epsilon\underset{=\dot{\balpha}(\epsilon)}{\left(\underbrace{\dot{\balpha}(0)+\epsilon\Ddot{\balpha}(0)}\right)}+\epsilon\underset{\dot{\balpha}(2\epsilon)}{\left(\underbrace{\dot{\balpha}(\epsilon)+\epsilon\Ddot{\balpha}(\epsilon)}\right)}+\dots+\epsilon\underset{\dot{\balpha}((n-1)\epsilon)}{\left(\underbrace{\dot{\balpha}((n-2)\epsilon)+\epsilon\Ddot{\balpha}((n-2)\epsilon)}\right)}.
    \end{eqnarray*}
We can now apply Equation \eqref{eq:approx-first} on all the first derivatives $\dot{\balpha}(j\epsilon)$ for $j\neq0$ again, and see that in total, $\epsilon\dot{\balpha}(0)$ appears $n$ times. Further, for each $j=0,\dots,n-2,$ the term $\epsilon^2\Ddot{\balpha}(j\epsilon)$ will appear $n-1-j$ times. 
We finally obtain \begin{equation}
    \balpha(n\epsilon)\approx\balpha(0)+\dot{\balpha}(0)+\epsilon^2\sum_{j=0}^{n-2}(n-1-j)\Ddot{\balpha}(j\epsilon).
\end{equation}
Our approximation is valid up to $\mathcal{O}(\epsilon^2).$ 
Equivalent derivations can be found in standard textbooks on ordinary differential equations, and a similar computation is done by \citet{smith2021origin}.

\section{Details of the Toy Experiments}
\label{app:experiments}

We provide the associated code to all experiments in the repository: \href{https://github.com/albertkjoller/geometric-ml/tree/main/reducing-memorisation-in-generative-models}{\texttt{github.com/albertkjoller/geometric-ml}}.

\subsection{Toy Experiment in 1D}

We adopt a Bayesian treatment of the parameters of $u_{\btheta}$, by defining the likelihood term in Equation \eqref{eq:lpost} as the flow matching loss $\mathcal{L}_{\operatorname{FM}}\left(\btheta\right)$ and a uniform prior over the weights. This corresponds to placing the Laplace approximation at the maximum likelihood estimate found from training the model with gradient descent. We consider the dataset $\mathcal{D}$ as a fixed collection of data-noise pairings and equidistant time samples, i.e. $\mathcal{D}=\{t_{i}=\frac{i}{N}, \bx_0^i, \bx_\ast^i\}_{n=1}^N$, for ensuring a deterministic loss.

We sample $S=1000$ model realisations from the Euclidean approximate posterior $\btheta_{\text{E-LA}}^s \sim q\left(\btheta|\mathcal{D}\right)$ and additionally use these as initial velocities for obtaining $S=1000$ samples from the Riemannian approximate posterior $\btheta_{\text{R-LA}}^s\sim q_{\mathcal{M}}\left(\btheta|\mathcal{D}\right)$. For both approximate posterior methods, we compute the velocity field over a space-time grid for each of the model realisations $\btheta_s$. We visualise the variation over the associated velocity fields in Figure \ref{fig:velocity-field} (\textit{right}).

Next, we sample $N=1000$ points from the base distribution and push these through the generator associated to every model realisation. This gives us a set of $N\times S$ generated points for each approximate posterior. We formalise this as:
\begin{align}
    \{\widehat{x}_{\text{E}}^i\}_{i=1}^{N\times S} &= \{ \ g_{\btheta_s}(x_0^i) \mid x_0^i \sim p_0,\ \btheta_\text{E}^s \sim q\left(\btheta|\mathcal{D}\right) \,\}_{i=1, s=1}^{N, S} \\
    \{\widehat{x}_{\text{R}}^i\}_{i=1}^{N\times S} &= \{ \ g_{\btheta_s}(x_0^i) \mid x_0^i \sim p_0,\ \btheta_\text{R}^s \sim q_{\mathcal{M}}\left(\btheta|\mathcal{D}\right) \,\}_{i=1, s=1}^{N, S}
\end{align}
and visualise the memorisation ratio for various distance thresholds $c$ in Figure \ref{fig:results} along with the distributions of generated samples, $p\left(\widehat{x}\right), p_{\text{E}}\left(\widehat{x}\right)$ and $p_{\text{R}}\left(\widehat{x}\right)$. We report the $\operatorname{KL}$-divergence to the target distribution $p_\ast$ as a measure of generalisation.

\subsubsection{Comparing Bias and Variance of the Generated Distributions}

\begin{figure*}[tb!]
    \centering
    \includegraphics[width=0.95\linewidth]{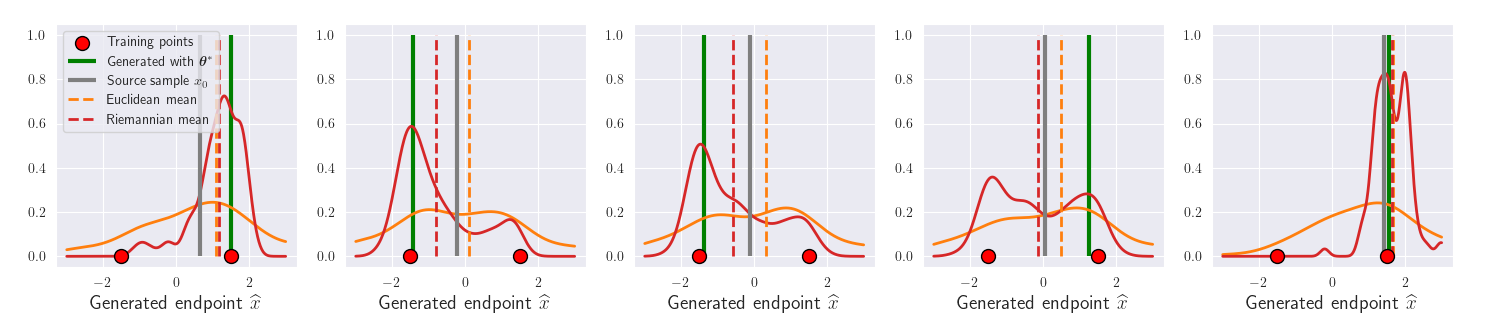}
    \caption{For a fixed initial value $\bx_0,$ (\grayline) we plot the distribution of endpoints of the trajectories when following the $S=1000$ vector fields under the Euclidean Laplace perturbation (\orangeline) and the Riemann Laplace perturbation (\redline), respectively. The green line (\greenline) demonstrates the endpoint of the original vector field that memorises the training points (\reddot). We use kernel density estimation to visualise the spread of the endpoints. We see that the Euclidean method has a bias and higher variance, whereas the Riemannian method finds the two modes and is often closer to the MAP endpoint.}
    \label{fig:bias}
\end{figure*}
In Figure \ref{fig:bias} we fix $N=1$ sample of $\bx_0$ and plot the mean and variance for the distribution of generated data samples obtained from the $S=1000$ generators, as argued in Subsection \ref{sub:perturbations-impact}. We observe that the Euclidean distribution has a bias and higher variance than the Riemannian distribution, as also seen in Figure \ref{fig:velocity-field}. Both findings indicate that the Riemannian paths stay closer to the MAP solution, as theoretically argued in Subsection \ref{sub:riemannian-paths}.

\subsection{Toy Experiment in 2D}\label{app:subsec:2d}

We consider the base distribution to be an isotropic Gaussian in $\rn^2$, i.e. $p_0\sim\mathcal{N}\left(\boldsymbol{0}, \mathbb{I}_2\right)$, and the target distribution $p_\ast$ to be a GMM with equally weighted means at $\bmu_1 = [1.5, 1.5]^\top$ and  $\bmu_2 = [-1.5, -1.5]^\top$, both with covariance $\bSigma = 0.2\cdot\mathbb{I}_2$. We train a network to memorise six points $\boldsymbol{P}_i$ in total, three at each mode, such that for all $i,j,$ $$\min_{k=1,2}\left\{\lVert\boldsymbol{P}_i-\bmu_k\rVert\right\}=\min_{k=1,2}\left\{\lVert\boldsymbol{P}_j-\bmu_k\rVert\right\}.$$ 
That is, the points $\boldsymbol{P}_i$ all have the same distance to whichever $\bmu_k$ they are closer to. A visualisation is shown in Figure \ref{fig:training_2dflow}.

\begin{figure*}[tb!]
    \centering
    \includegraphics[width=0.3\linewidth]{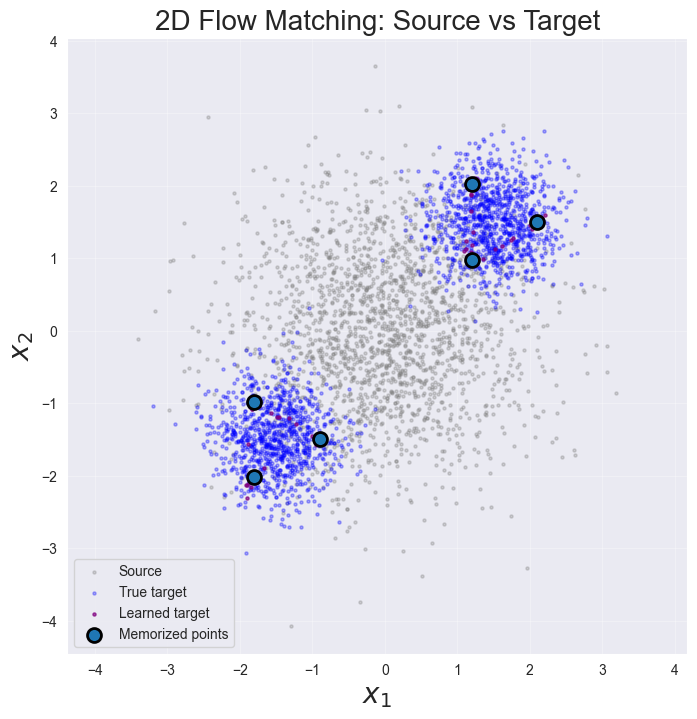}
    \caption{The source distribution $p_0$ samples are shown in grey, and the target distribution $p_\ast$ (\kleinbluedot) samples in blue. The model learns to memorise the equidistant points $\boldsymbol{P}_i$ (\bluedot) from each mixture mode. In purple (\purpledot), we show samples from the learnt distribution.}
    \label{fig:training_2dflow}
\end{figure*}

We sample the base distribution as in the 1D toy example in Section \ref{sec:experiments} and choose the initial velocities to be along the eigenvector of the Hessian $\mathbf{H}_{\mathcal{L}}\left(\btheta\right)$ corresponding to the highest eigenvalue, which intuitively corresponds to the most impactful direction to travel along the loss manifold. We explore different scaling levels $\eta$ of the initial velocity vectors. The trajectories from $p_0$ to $p_\ast$ given by the Euclidean Laplace and the Riemannian perturbation are shown in Figure \ref{fig:2d_higheigvec_trajectories}.

\begin{figure*}[tb!]
    \includegraphics[width=0.98\linewidth]{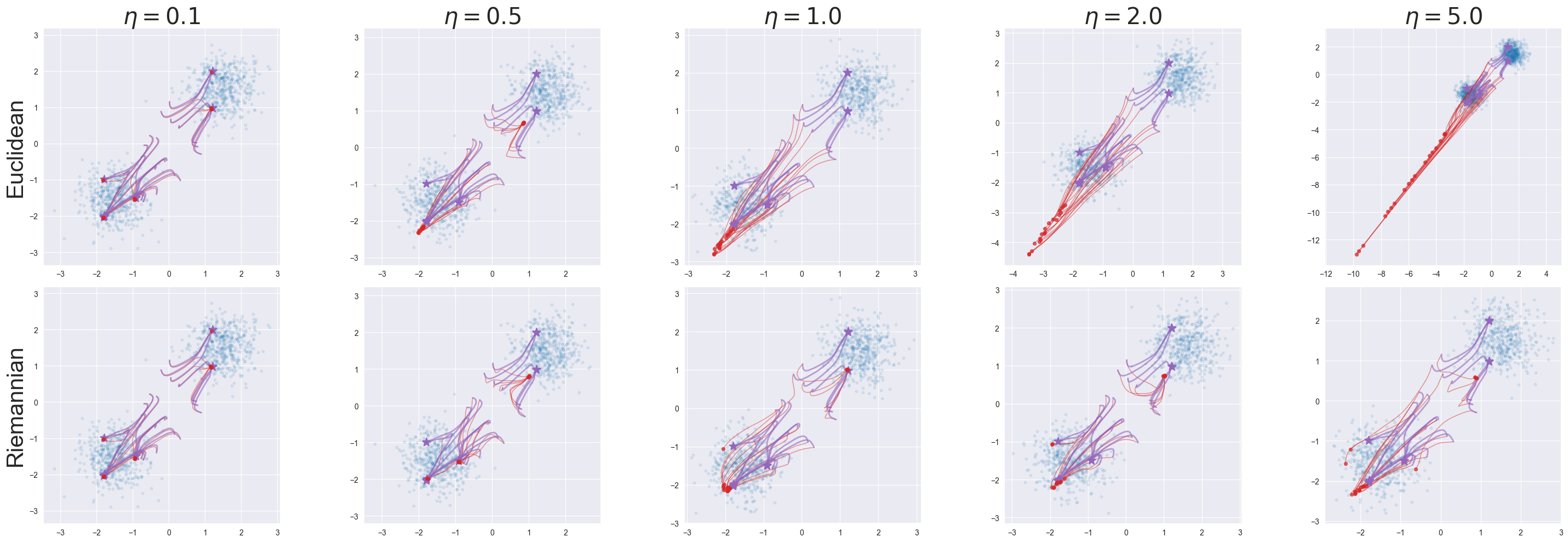}
    \hfill
    \caption{The blue points denote samples from the true target distribution $p_\ast$, and the purple lines are trajectories from the base distribution $p_0$. The endpoints of the learnt model (MAP estimate) are represented by (\purplestar). The red lines correspond to the trajectories from the perturbed models (top row corresponding to linear and bottom to the geodesic perturbations), ending at $\widehat{\bx}=g_{\btheta^s}(\bx_0), \bx_0\sim p_0$ (\reddot). As the magnitude of the initial velocity $\bv$ increases, the geodesic trajectories remain within favourable regions of the target space, whereas the Euclidean samples render trajectories that fall into low-density regions of $p_\ast$.}
    \vspace{-9pt}
    \label{fig:2d_higheigvec_trajectories}
\end{figure*}

Furthermore, we then sample 3000 points from $p_0$ and examine the endpoints following the Euclidean and Riemannian Laplace approximations, and MAP models by quantifying the empirical Wasserstein distribution between samples from the true target $p_\ast$ and the generated samples $g_{\btheta^s}(\bx_0), \bx_0\sim p_0$. 
Recall the definition of the Wasserstein distance between two probability measures. 
\begin{definition}\cite{villani2008optimal}
    Let $\boldsymbol{\mu}$ and $\boldsymbol{\nu}$ be probability measures on a Polish metric space $(\mathcal{X},d).$ Then the Wasserstein 1-distance between $\boldsymbol{\mu}$ and $\boldsymbol{\nu}$ is given by 
    \begin{equation}
        W_1(\boldsymbol{\mu},\boldsymbol{\nu})=\inf_{\pi\in\Pi(\boldsymbol{\mu},\boldsymbol{\nu})}\int_{\mathcal{X}}d(x,y)d\pi(x,y)=\inf\left\{\mathbb{E}d(X,Y), \ X\sim\boldsymbol{\mu}, \ Y\sim\boldsymbol{\nu}\right\}.
    \end{equation}
\end{definition}
Following \citet{peyre2019computational}, this can be adapted to empirical measures $\boldsymbol{\mu}=\frac{1}{n}\sum_{i=1}^n\delta_{\bx_i}$ and $\boldsymbol{\nu}=\frac{1}{m}\sum_{j=1}^m\delta_{\boldsymbol{y}_j}$ to obtain the Wasserstein $1$-distance
\begin{equation}
W_1(\boldsymbol{\mu},\boldsymbol{\nu}) = \inf_{\gamma \in \Pi(\boldsymbol{\mu},\boldsymbol{\nu})}\left\{ \sum_{i=1}^{n} \sum_{j=1}^{m} \gamma_{ij} \lVert\bx_i - \boldsymbol{y}_j\rVert, \quad \sum_{j=1}^{m} \gamma_{ij} = \frac{1}{n}, \quad \sum_{i=1}^{n} \gamma_{ij} = \frac{1}{m}, \quad \gamma_{ij} \geq 0\right\}
\end{equation}
by choosing $d(x,y)=\lVert x-y\rVert_2$ to be the standard Euclidean distance on $\rn^2.$

\begin{figure*}[tb]
    \includegraphics[width=0.97\linewidth]{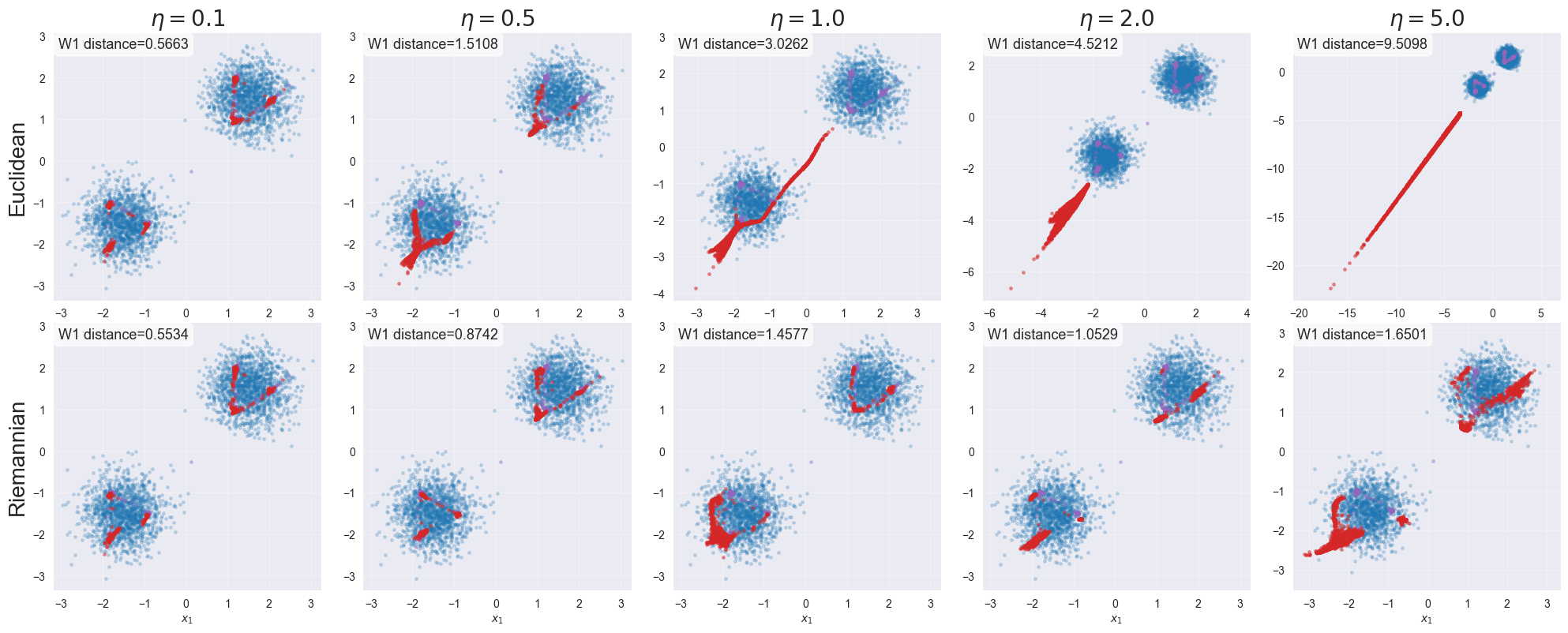}
    \hfill
    \caption{We compare 3000 points $\widehat{\bx}=g_{\btheta^s}(\bx_0), \bx_0\sim p_0$ (\reddot) generated by the models sampled with the Euclidean Laplace approximation (top row) and the Riemannian Laplace approximation (bottom row) and show them in superposition with the samples generated by the MAP model $g_{\btheta^*}$ (\purpledot). The distribution of the points generated by the MAP model and the target distribution are at a Wasserstein 1-distance of $0.4348$. We see that by increasing initial velocity scaling, points from the Euclidean setting fall in regions of low mass of $p_\ast$, and their distribution increases in Wasserstein distance to the $p_\ast$. Using the Riemannian Laplace approximation, the effect is remarkably milder.}
    \label{fig:W1_2dflow}
\end{figure*}

\newpage
\section{Laplace Approximations for Generative Image Models}
\label{app:image-experiments}

In this section, we show how we analysed memorisation exhibited by generative image models. We summarise our approach in Algorithm \ref{alg:riemannian-bayesian-inference}, and provide details in the following subsections.

\subsection{Model Details}
We use a time-conditioned 2D U-Net architecture with residual connections and a self-attention bottleneck, as commonly used in diffusion models \cite{ho2020denoising}\footnote{Implementation taken from \url{https://huggingface.co/google/ddpm-cifar10-32}}. The network follows an encoder–decoder U-Net design with four downsampling and four upsampling stages based on $3\times3$ convolutions with padding and striding of shape $(1\times 1)$. The stages are connected by skip connections. The network processes RGB inputs, conditions all residual blocks on a learned 512-dimensional time step embedding, and uses \texttt{GroupNorm} with \texttt{SiLU} activations throughout.

We found that a pre-trained model trained on the full dataset showed no memorisation. Therefore we choose a subset of $1000$ randomly sampled images from CIFAR-10, normalised to the range $[-1, 1]$ and train the model for up to $5000$ epochs using a learning rate of $2\cdot 10^{-4}$ using the denoising diffusion probabilistic model (DDPM) scheduler with $1000$ steps. We store checkpoints every $50$th epoch and later generate images with $50$ DDIM steps, using the model that we deem to balance memorisation and generalisation best, defined by the Frechét Inception Distance (FID) \citep{heusel2017gans}.

\begin{algorithm*}[tb!]
    \caption{Riemannian Bayesian Inference for Generative Models in High Dimensions}\label{alg:riemannian-bayesian-inference}

    \begin{algorithmic}[1]
        \STATE {\bfseries Input:} Generator $g_{\btheta}$ with $K$ parameters, twice-differentiable loss $\mathcal{L}(\btheta)$, training data $\mathcal{D}_{\bx}$, batch size $N$, number of perturbations $S$, perturbation strength $t$,  repetitions $B$, number of Lanczos iterations $k\ll K$.
        \vspace{0.3em}
        \STATE {\bfseries Output:} Generated images $\{\widehat{\boldsymbol{x}}_{n,s,b}\}_{n=1,s=1,b=1}^{N,S,B}$
        \vspace{0.7em}
    
        \STATE {\bfseries Step 1: MAP training}
        \STATE \hspace{1em} $\btheta^\ast \gets \arg\min_\btheta \mathcal{L}(\btheta)$ \hfill \COMMENT {Train $g_{\btheta}$ using SGD}
        \vspace{0.5em}
    
        \FOR{$b = 1$ {\bfseries to} $B$}
            \STATE Fix a batch of $N$ training data $\boldsymbol{X}_\ast^{(b)}=\{\bx_{\ast}^n\}_{n=1}^N \sim \mathcal{D}_{\bx}$, noise samples $\boldsymbol{X}_0^{(b)}=\{\bx_{0}^n\}_{n=1}^N$, and time samples $\boldsymbol{t}^{(b)}$
            \STATE Define batch loss: $\mathcal{L}^{(b)}(\btheta) := \mathcal{L}(\btheta \mid \boldsymbol{X}_\ast^{(b)}, \boldsymbol{X}_0^{(b)}, \boldsymbol{t}^{(b)})$
            \vspace{0.5em}
    
            \STATE {\bfseries Step 2: Hessian approximation} \hfill \COMMENT {Compute exact Hessian if feasible}
            \STATE  Approximate Hessian $\mathbf{H}_{\mathcal{L}}:=\nabla_\btheta^2 \mathcal{L}^{(b)}(\btheta^\ast)
                \approx \boldsymbol{Q} \boldsymbol{T} \boldsymbol{Q}^\top$ using $k$ Lanczos iterations.
            \STATE Eigendecompose  $\boldsymbol{T} = \boldsymbol{V} \boldsymbol{\Lambda} \boldsymbol{V}^\top$ and restrict to positive pairs $\boldsymbol{\Lambda}_{+}$ and $\boldsymbol{V}_{+}$
            
            \STATE  Define $q\left(\btheta\right) = \mathcal{N}\left(\btheta, \mathbf{H}_{\mathcal{L}}^{-1}\right)$ using approximation $\mathbf{H}_{\mathcal{L}}^{-1} \approx \boldsymbol{Q} \boldsymbol{V}_{+} \boldsymbol{\Lambda}_{+}^{-1} \boldsymbol{V}_{+}^\top \boldsymbol{Q}^\top$ \hfill \COMMENT {Laplace approximation}
            \vspace{0.5em}

            \STATE {\bfseries Step 3: Parameter perturbation}
            \FOR{$s = 1$ {\bfseries to} $S$}
                \STATE Compute perturbation from the Laplace approximation: $\boldsymbol{v} \gets \boldsymbol{Q} \boldsymbol{V}_{+} \boldsymbol{\Lambda}_{+}^{-1/2} \boldsymbol{V}_{+}^\top \boldsymbol{Q}^\top\boldsymbol{\epsilon}$ \hfill \COMMENT {Sample $\boldsymbol{\epsilon} \sim \mathcal{N}(\boldsymbol{0}_k, \mathbb{I}_k)$}
    
                \STATE Compute parameter perturbation: \vspace{0.1em}
                \STATE \hspace{1em} {\bfseries Euclidean:} \quad $\btheta_{s,b} \gets \btheta^\ast + t\,\boldsymbol{v}$
                \STATE \hspace{1em} {\bfseries Riemannian:} $\bigl(\btheta_{s,b}, \mathcal{L}(\btheta_{s,b})\bigr) \gets 
                \operatorname{Exp}_{h(\btheta^\ast)}(t\cdot\mathbf{J}_h(\btheta^*)\bv)$
                \vspace{0.5em}

                \STATE {\bfseries Step 4: Data generation}
                \STATE Generate $\widehat{\boldsymbol{x}}_{n,s,b} \gets g_{\btheta_{s,b}}(\boldsymbol{x}_0^n)$ from noise samples $\boldsymbol{x}_0^n \in \boldsymbol{X}_0^{(b)}$.
                \vspace{0.4em}
            \ENDFOR
        \ENDFOR
    \end{algorithmic}
\end{algorithm*}

\subsection{Hessian Approximation in High Dimensions}
Recall that the Laplace approximation is defined as a Gaussian $\mathcal{N}(\btheta^\ast, \mathbf{H}^{-1})$ where $\btheta^\ast$ is the MAP and $\mathbf{H}^{-1}:=\mathbf{H}_{\mathcal{L}}^{-1}\left(\btheta^\ast\right)$ is the inverse Hessian at the MAP. It should be noted that the inverse Hessian must be positive definite to be a valid covariance matrix. If so, we can sample from the approximate posterior using a reparameterisation trick:
$$
    \btheta^s = \btheta^\ast + \mathbf{\Sigma}^{1/2} \boldsymbol{\epsilon},  \qquad \boldsymbol{\epsilon} \sim \mathcal{N}\left(\boldsymbol{0}, \mathbb{I}_K\right).
$$
In the context of deep neural networks, we face two problems related to drawing samples like this. First, the Hessian is in practice too large to keep and manipulate in memory. Second, the Hessian of a deep neural network is generally indefinite.

\begin{figure}[tb]
    \centering
    \includegraphics[width=\linewidth]{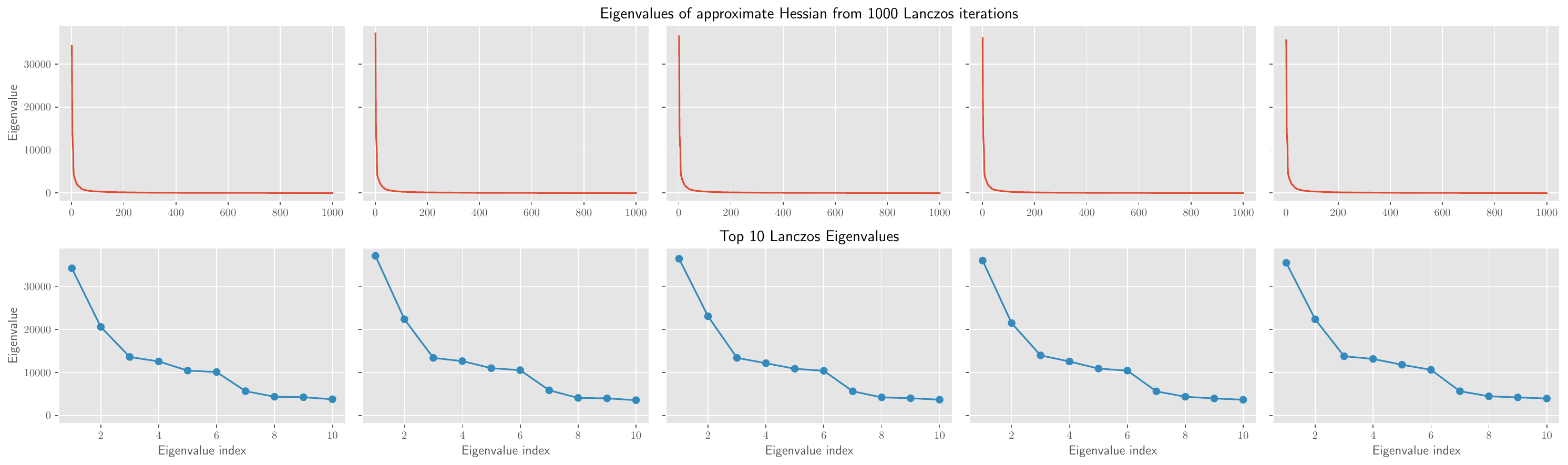}
    \caption{Eigenspectra of approximated Hessians, computed for $5$ different batches of training data.}
    \label{fig:some-hessian-spectra}
\end{figure}

We solve the first problem by using a low-rank approximation of the Hessian, namely $\mathbf{H}\approx \boldsymbol{Q}\boldsymbol{T}\boldsymbol{Q}^\top$ where $\boldsymbol{Q}\in \mathbb{R}^{K\times k}$ and $\boldsymbol{T}\in \mathbb{R}^{k\times k}$ is a tridiagonal matrix. We determine these matrices with the matrix-free Lanczos iteration \citep{lanczos1950iteration} with full re-orthogonalisation using $k=1000$ Lanczos iterations. This approach only requires Hessian-vector products which are efficiently computed using automatic differentiation, and does not require storing the full Hessian in memory. Although the procedure in theory works for approximating the full Hessian of large networks, we note that the full re-orthogonalisation procedure requires keeping $k$ vectors of $K$ dimensions in the matrix $\mathbf{Q}$ in memory, which is not feasible for the full network.

To address the second issue, note that we need the matrix square root of the inverse Hessian for sampling the Laplace posterior. We compute this through an eigendecomposition of the tridiagonal matrix, hence $\boldsymbol{T}=\boldsymbol{V} \boldsymbol{\Lambda}\boldsymbol{V}^\top$ where $\boldsymbol{V}$ contains the eigenvalues of $\boldsymbol{T}$ which approximate the extremal eigenvalues of $\mathbf{H}$. Note, that this allows us to approximate the Hessian and its inverse by
$$
    \mathbf{H} \approx \boldsymbol{Q} \boldsymbol{V} \boldsymbol{\Lambda}\boldsymbol{V}^\top \boldsymbol{Q}^\top,\qquad \mathbf{H}^{-1} \approx \boldsymbol{Q} \boldsymbol{V} \boldsymbol{\Lambda}^{-1}\boldsymbol{V}^\top \boldsymbol{Q}^\top.
$$
As mentioned, the Hessian can easily be indefinite and so can its approximation. We deal with this in the naive way by removing all non-positive eigenvalues and their associated eigenvectors. We denote these by $\boldsymbol{V}_{+}$ and $\boldsymbol{\Lambda}_{+}$, and remark that the resulting Hessian approximation is positive definite. It is well known that every positive definite matrix is invertible, and its inverse is also positive definite. This further implies the existence of the square root of the inverse, in our case we obtain:
$$
    \mathbf{H}_{+}^{-1/2} \approx \boldsymbol{Q} \boldsymbol{V}_{+} \boldsymbol{\Lambda}^{-1/2}_{+}\boldsymbol{V}_{+}^\top \boldsymbol{Q}^\top.
.
$$
This allows us to sample the velocity $$\bv=\boldsymbol{Q} \boldsymbol{V}_{+} \boldsymbol{\Lambda}^{-1/2}_{+}\boldsymbol{V}_{+}^\top \boldsymbol{Q}^\top \boldsymbol{\epsilon}, \quad \boldsymbol{\epsilon}\sim\mathcal{N}(0,\mathbb{I}_K), $$
which is used for computing both the Euclidean and Riemannian perturbation.

\subsection{Generating Images with Perturbed Parameters on Stochastic Loss Manifolds}

We are working with a high-dimensional model and a large dataset, hence we must resort to mini-batching, which means that we use a \textit{stochastic} version of our method. We do so by defining a posterior loss manifold (as described in Subsection \ref{subsec:RLA}) per batch of the training data, $\mathcal{D}_i \sim \mathcal{D}$. We additionally fix the initial samples $\bx_0 \sim p_0$ and the time step samples to ensure that stochasticity only comes from the batching operation.

We construct a manifold for $B$ batches and follow the strategy of approximating the Hessian described above, resulting in $B$ Riemannian Laplace approximations from which we can draw parameter perturbations. We plot the approximate eigenspectra for $5$ such batches in Figure $\ref{fig:some-hessian-spectra}$, using a batch size of 64. Although the spectra differ slightly, they all reflect that the Hessians are dominated by few highly curved directions and many near-zero directions. 

To summarise, we perturb the model parameters by 
\begin{enumerate}
    \item defining the manifold for a data batch,
    \item approximating the Hessian of this manifold,
    \item solving the Riemannian perturbation with $S$ initial velocity samples as explained in the previous subsections,
    \item generating $N$ images for each sample $\btheta^s$.
\end{enumerate}
By repeating for $B$ batches, we get $B\times S \times N$ generated images from which we can do analyses of images generated over all $B$ batches or include uncertainty by studying batch-wise variations.

\subsubsection{Sensitivity Issues in the Geodesic Computation} We observe that solving the geodesic equation is sensitive to the numerical integrator used, as well as the magnitude of the sampled initial velocity vector. First, note that we only get valid results when using an adaptive step-size method, e.g. 5th-order Runge-Kutta method \cite{dormand1980family}, while fixed step size methods like the Euler scheme or the 4th-order Runge-Kutta with a 3/8 rule result in generating psychedelic noise images using the perturbed parameters. We hypothesise that this is due to the adaptive method being able to slow down when approaching the more curved directions. Secondly, for initial vectors of large magnitude, the numerical solution of the geodesic equation can be slow to converge, as the convergence threshold needs to be sufficiently low for the numerical integrator to find a good solution. In the interest of time, we choose to filter out samples displaying this behaviour, by setting a stopping time criteria when solving the geodesic equation.

\subsection{Memorisation Results}

We consider a single batch of $64$ training samples and compute $S=10$ perturbations according to the previous descriptions. For each perturbation we generate $N=100$ images, resulting in a total of $1000$ generated images from the ensemble. We compute the memorisation ratio with respect to the $|\mathcal{D}_i|=1000$ randomly chosen training points, subject to a memorisation constant of $c=1/3$. We extend the memorisation measure to consider the $K=10$ next-closest images from the training set, rather than only the second-closest, since there are image duplicates in the training data. Similarly to the 2D case, we introduce a scaling parameter $\eta$ when sampling the approximate posterior, however, we now do not limit samples to be along the top-eigenvector.

As seen in Figure \ref{fig:memorisation-images}, roughly $90\%$ of the images generated using the MAP model are memorised. We also see the effect of perturbing the model parameters, in the sense that it reduces memorisation, no matter the perturbation strategy. Remarkably, the Riemannian perturbations reduce memorisation more than the Euclidean ones when increasing the magnitude of the initial velocity samples $\eta$. This surprised us, and we believe that it relates to the sensitivity of the numerical integrator, as described in the previous subsection. 

When looking at a random selection of generated images using the MAP model, we clearly see why the memorisation ratio is that high (Figure \ref{fig:image-examples-perturbation}). Using the same initial samples $\bx_0 \sim p_0$, we additionally see that perturbing the parameters using e.g. $\eta=0.5$ results in noisier images, which is exactly a valid way to reduce memorisation according to its definition. Conceptually, one might however argue that the image concepts are still memorised. We hypothesise that this is related to the high memorisation rate of the MAP model, suggesting that it converged to a local minima which might be hard to escape from using the geodesics. 

A MAP solution that memorises effectively learns to recreate the training data. Consequently, all models based on variations of the MAP are also able to recreate the training data, but the difference is that such models additionally introduce new characteristics that were not present in the training data. In the image domain, these characteristics should ideally be structured so as to meaningfully alter the underlying concept of the training image, for example by changing the time of day, the background, or the shape of the pictured object. However, while our approach can generate new images in the vicinity of the training data, it is not currently optimised for images and therefore does not constrain the new samples to lie on the image manifold. We hope to address this issue soon to further improve our method on the image domain.
\begin{figure}[h!]
    \centering
    \includegraphics[width=0.6\linewidth]{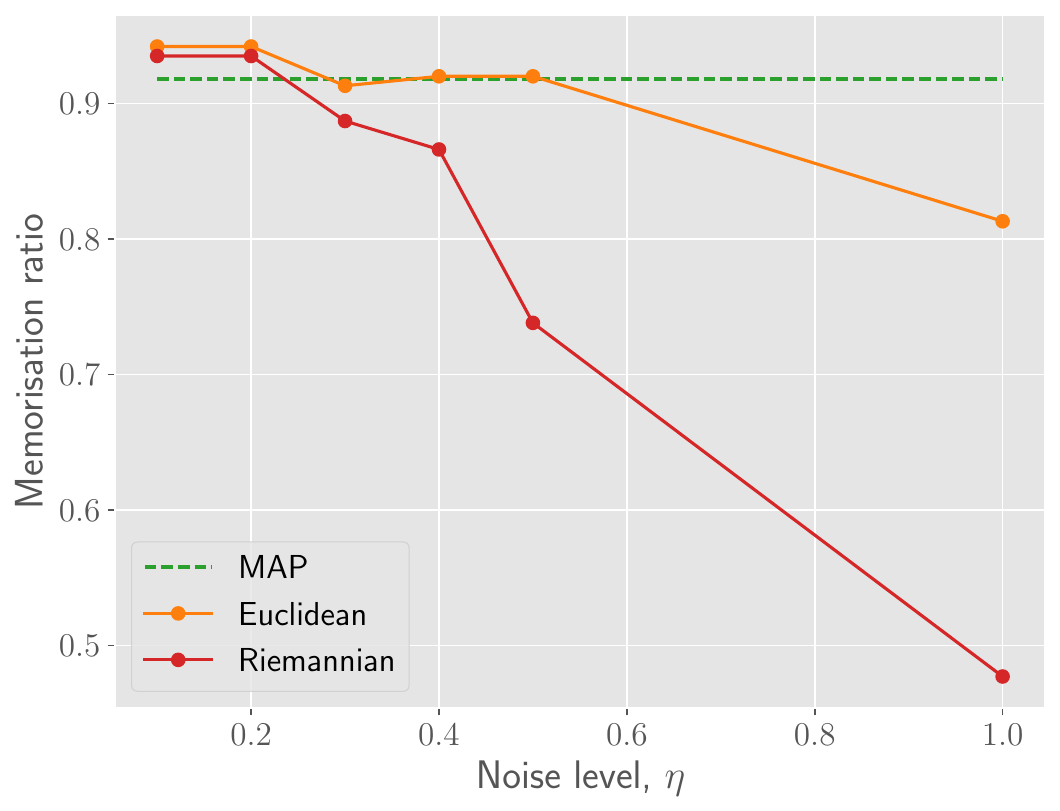}
    \caption{Memorisation reduces by perturbing the model parameters for a generative image model. This happens for both types of Laplace approximations defined from a single batch of training samples.}
    \label{fig:memorisation-images}
\end{figure}

\begin{figure}[h!]
    \centering
    \begin{subfigure}[b]{0.95\textwidth}
        \centering
        \includegraphics[width=\linewidth]{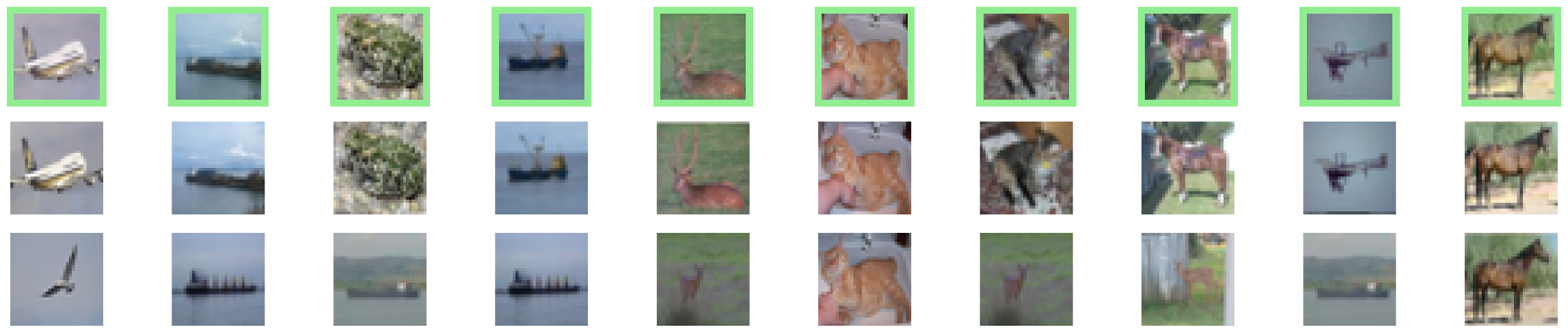}
        \caption{MAP}
    \end{subfigure}
    \begin{subfigure}[b]{0.95\textwidth}
        \centering
        \includegraphics[width=\linewidth]{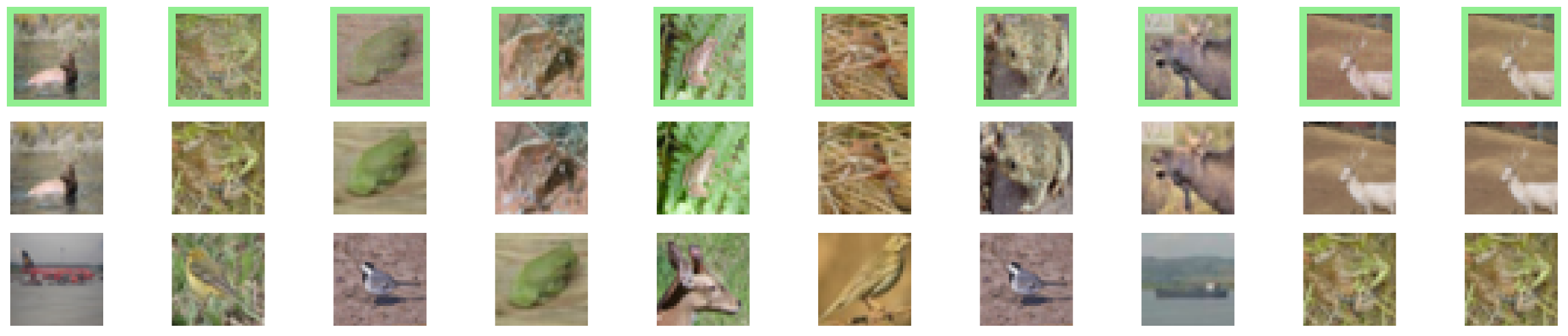}
        \caption{Euclidean, $\eta=0.5$}
    \end{subfigure}
    \begin{subfigure}[b]{0.95\textwidth}
        \centering
        \includegraphics[width=\linewidth]{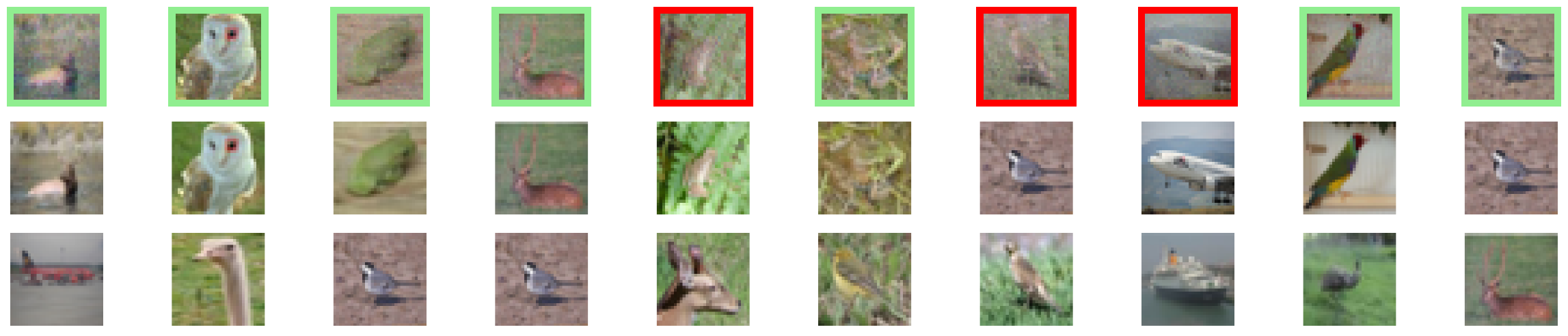}
        \caption{Riemannian, $\eta=0.5$}
    \end{subfigure}
    \caption{We show a random selection of generated images where each column uses the same initial sample $\bx_0\sim p_0$. For each of the three figures the top row corresponds to the generated images. The second and third row corresponds to the closest and second-closest images from the subsampled training set, respectively. Using a memorisation criteria of $c=1/3$ and based on the $K=10$ next-closest neighbours, we indicate whether a generated sample is memorised by framing it with a green box. If it is not memorised, we frame it by a red box. The Riemannian samples are memorised less often than the other approaches, which seems to only be due to added noise on the pixels. We hypothesise that the result depends on properties of the loss at the MAP and hence the MAP memorisation ratio.}
    \label{fig:image-examples-perturbation}
\end{figure}

\newpage
\appendix
\onecolumn

\end{document}